\documentclass{article}

\PassOptionsToPackage{table}{xcolor}
\usepackage{etoolbox}       %
\usepackage[utf8]{inputenc} 

\usepackage{graphicx}  %

\usepackage[parfill]{parskip}
\usepackage{amsmath,amsthm}
\usepackage{mathtools}  %
\usepackage{adjustbox}
\newtoggle{arxiv}
\toggletrue{arxiv}

\newcommand{\vv}{\mathbf{v}}

  \usepackage[parfill]{parskip}
  \usepackage[citestyle=authoryear-comp,sorting=nyt,maxbibnames=99, maxcitenames=1, backend=biber, uniquename=false]{biblatex}
  \newcommand{\citep}{\parencite}
  
  \newcommand{\citet}{\textcite}
  \newlength{\defbaselineskip}
  \RequirePackage[T1]{fontenc}
  \RequirePackage[tt=false, type1=true]{libertine}
  \RequirePackage[varqu]{zi4}
  \RequirePackage[libertine]{newtxmath}

\usepackage{bm}

\usepackage[inline]{enumitem}
\usepackage{booktabs}
\usepackage{subcaption}
\usepackage{multirow}
\usepackage{wrapfig}
\usepackage{algorithm,algorithmicx,algpseudocode}
\usepackage{nicematrix}

\usepackage[frozencache,cachedir=.]{minted} %

\usepackage{import}
\usepackage{lettrine}

\newtheorem{dfn}{Definition}

\usepackage{pifont}%
\newcommand{\xmark}{\ding{55}}%
\usepackage[T1]{fontenc}

\usepackage{thm-restate}

\usepackage{url}            
\usepackage{booktabs}       
\usepackage{amsfonts}       
\usepackage{nicefrac}       
\usepackage{microtype}      

\usepackage{wrapfig}

\usepackage{caption}
\usepackage{capt-of}
\usepackage{lipsum}

\usepackage[colorlinks]{hyperref}
\usepackage[capitalise,noabbrev]{cleveref}

\usepackage{threeparttable}

\definecolor{c1}{HTML}{586770}
\definecolor{c4}{HTML}{2a4a67}
\definecolor{c3}{HTML}{6d2a58}
\definecolor{c2}{HTML}{34142a}
\hypersetup{colorlinks=true, citecolor=c3, linkcolor=c2, urlcolor=c2}

\definecolor{myblue}{HTML}{FDF5E0} 
\definecolor{mygray}{HTML}{DBE2E9} 
\definecolor{mygreen}{HTML}{E6F3FC}

\definecolor{dark2orange}{rgb}{0.9, 0.4, 0.}
\definecolor{dark2purple}{rgb}{0.4, 0.4, 0.8}

\hypersetup{
    colorlinks=true,
    linkcolor=black,
    urlcolor=c1,
    citecolor=c1,
}

\usepackage{multirow}

\newcommand{\undermath}[2]{\underset{#1}{\underbrace{#2}}}

\newcommand{\R}[0]{\mathbb{R}}

\newcommand{\inner}[2]{\langle #1, #2 \rangle}

\newcommand{\M}[0]{\mathcal{M}}

\newcommand{\head}[1]{\vspace{1.7mm}\noindent{{\textcolor{c4}{\bf #1.}}}}

\newcommand{\model}[0]{\textsc{Atlas}}
\newcommand{\omodel}[0]{\textsc{OmegaNet}}
\newcommand{\learningrule}[0]{Omega}

\newcommand{\mb}[1]{\mathbf{#1}}

\newcommand{\SSS}{\mathcal{S}}

\usepackage{mathtools}

\usepackage{tikz}

\newcommand{\cL}{\mathcal{L}}

\newcommand{\rank}{\operatorname{rank}}

\usepackage{authblk}
\usepackage{epigraph}
\usepackage{soul}

\newcommand{\vk}{\mathbf{k}}
\newcommand{\vq}{\mathbf{q}}
\renewcommand{\vv}{\mathbf{v}}

\newcommand{\vx}{\mathbf{x}}

\usepackage{tcolorbox}
\usepackage{xcolor}
\usepackage{lettrine,Zallman}
\newtcolorbox{c4box}{boxrule=1pt, colback=c4!5!white,colframe=c4!50!white}
\newtcolorbox{c3box}{boxrule=1pt, colback=c3!5!white,colframe=c3!50!white}
\newtcolorbox{c2box}{boxrule=1pt, colback=c2!5!white,colframe=c2!50!white}
\newtcolorbox{c1box}{boxrule=1pt, colback=c1!5!white,colframe=c1!50!white}

\newcommand\blfootnote[1]{%
  \begingroup
  \renewcommand\thefootnote{}\footnote{#1}%
  \addtocounter{footnote}{-1}%
  \endgroup
}

\title{\model: Learning to Optimally Memorize the Context at Test Time}

\author{Ali Behrouz}
\author{Zeman Li}
\author{Praneeth Kacham}
\author{Majid Daliri}
\author{Yuan Deng}
\author{Peilin Zhong}
\author{\\Meisam Razaviyayn} 
\author{Vahab Mirrokni \protect \blfootnote{\texttt{\{alibehrouz, zemanli, pkacham, dengyuan, peilinz, razaviyayn, mirrokni\}@google.com}, and majiddl.2099@gmail.com}}

\affil{\protect\includegraphics[width=40mm]{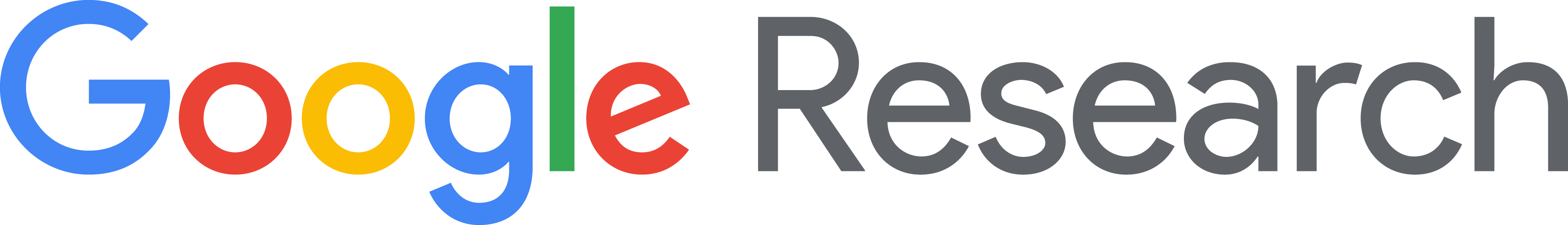}} 
\date{}

\begin{document}

\vspace{-20ex}
\maketitle

\begin{abstract}
 Transformers have been established as the most popular backbones in sequence modeling, mainly due to their effectiveness in in-context retrieval tasks and the ability to learn at scale. Their quadratic memory and time complexity, however, bound their applicability in longer sequences and so has motivated researchers to explore effective alternative architectures such as modern recurrent neural networks (a.k.a long-term recurrent memory module). Despite their recent success in diverse downstream tasks, they struggle in tasks that requires long context understanding and extrapolation to longer sequences. We observe that these shortcomings come from three disjoint aspects in their design: (1) limited memory capacity that is bounded by the architecture of memory and feature mapping of the input; (2) online nature of update, i.e., optimizing the memory only with respect to the last input; and (3) less expressive management of their fixed-size memory. To enhance all these three aspects, we present \model, a long-term memory module with high capacity that learns to memorize the \emph{context} by optimizing the memory based on the current and past tokens, overcoming the online nature of long-term memory models. Building on this insight, we present a new family of Transformer-like architectures, called \textsc{DeepTransformers}, that are strict generalizations of the original Transformer architecture. Our experimental results on language modeling, common-sense reasoning, recall-intensive, and long-context understanding tasks show that \model{} surpasses the performance of Transformers and recent linear recurrent models. \model{} further improves the long context performance of Titans, achieving +80\% accuracy in 10M context length of BABILong benchmark. 
\end{abstract}

\section{Introduction}\label{sec:intro}
The attention module~\citep{bahdanau2014neural} is a critical building block in modern deep learning architectures~\citep{transformers, achiam2023gpt, behrouz2024titans, team2025gemma}, excelling due to its scalability and performance in in-context retrieval tasks. In principle, attention functions as an associative memory, computing direct pairwise token dependencies to store key-value mappings and retrieve them via query-key similarities. Computing this  pairwise dependencies, however, while accurate, causes quadratic space and time complexity, limiting their applicability in long context understanding, memorization, or modeling~\citep{liu2024lost, li2024survey, dalal2025one}.


Recent research efforts aim to overcome the limitations of Transformers—i.e., pure attention-based architectures—in long-context modeling by designing more efficient yet effective recurrent neural networks~\citep{schlag2021linear, behrouz2024titans, peng2025rwkv7}. These modern recurrent architectures can be unified as associative memory modules optimizing an internal objective termed 'attentional bias'~\citep{behrouz2025Miras}. Unlike Transformers' growing KV cache, these models use fixed-size memory, necessitating improved memory management. Consequently, there's growing interest in enhancing RNN memory management through more effective: (i) Learning rules, from additive learning~\citep{katharopoulos2020transformers} to DeltaNet's Delta rule~\citep{schlag2021linear}; (ii) Forget (Retention) Gates, from RetNet's input-independent gating~\citep{sun2023retentive} to adaptive gating in Titans~\citep{behrouz2024titans} and RWKV-7~\citep{peng2025rwkv7}; and (iii) Memory Architectures, from vector-valued memory~\citep{sun2023retentive, peng2023rwkv} to neural deep memory modules~\citep{behrouz2024titans, sun2024learning}.

Despite the success of these improved models in a diverse set of downstream benchmarks, they often struggle with long context understanding, in-context retrieval, and extrapolation to longer sequences~\citep{wen2024rnns, behrouz2024titans, arora2024simple, yang2024gated}. We observe these shortcomings arise from three design aspects: (1) The online nature of their memory update, where memory is optimized based on the current token while retaining past memory state, leading to memorization of individual tokens without considering broader context; (2) The limited capacity of memory, where architecture and key-value feature mappings restrict the number of perfectly mappable key-value pairs; and (3) The expressiveness of memory management (i.e., the internal objective's optimizer), as most recent models use gradient descent that relies on the first-order information about the dynamics of tokens, causing the memory to converge to spurious local minima and learn less effective key-value mappings.

\begin{table*}
    \centering
    \caption{A summary of the recent modern recurrent neural networks. We compare these architectures based on five characteristics: (1) Dynamic decay; (2) Deep neural memory; (3) non-linear memory capacity; (4) Locally optimal: managing memory by (approximating) the second-order information about tokens; (5) Flexible context: the ability to flexibly memorize the \underline{context}. $\phi(\cdot)$ and $\phi^{*}(\cdot)$ represent polynomial and infinite-dimensional feature mappings (see \autoref{eq:infi-map}).}
    \hspace*{-3ex}
    \resizebox{1.05\linewidth}{!}{
    \begin{tabular}{l l c c c c c c l}
    \toprule
         \multirow{2}{*}{Model} & \multirow{2}{*}{Attentional Bias  $\ell(\cdot; \cdot)$} & \multirow{2}{*}{Optimizer} & \multirow{1}{*}{Dynamic} & \multirow{1}{*}{Deep} & Non-linear & \multirow{1}{*}{Locally} & \multirow{1}{*}{Flexible}  & \multicolumn{1}{c}{\multirow{2}{*}{Memory Write Operation}} \\
          &  &  & Decay & Memory & Capacity$^{\dagger}$ & Optimal & Context  &  \\
         \midrule
         \midrule
         Attention & $\sum_{t = 1}^{L} a_t \| \M \vk_t - \vv_t \|^2_2$ & NP$^\ddagger$& \xmark & \xmark & \checkmark & \checkmark & \xmark & $\M_t = \M_{t-1} \cup \{(\vk_t, \vv_t)\}$\\
         SWA  & $\sum_{t = c}^{L} a_t \| \M \vk_t - \vv_t \|^2_2$ & NP & \xmark & \xmark & \checkmark & \checkmark & \checkmark  & $\M_{t} = (\M_{t-1} \setminus \{(\vk_c, \vv_c)\}) \cup \{(\vk_t, \vv_t)\}$\\
         \midrule
         Linear Attention &  $\inner{\M_t \vk_t}{\vv_t}$ & GD & \xmark & \xmark & \xmark & \xmark & \xmark & \multicolumn{1}{l}{$\M_t =    \M_{t-1} +  {\vv_t  \vk_t^\top}$}\\
         RetNet &  $\inner{\M_t \vk_t}{\vv_t}$ & GD & \xmark & \xmark & \xmark & \xmark & \xmark  & \multicolumn{1}{l}{$\M_t =   \alpha \M_{t-1} +  {\vv_t  \vk_t^\top}$}\\
         GLA & $\inner{\M_t \vk_t}{\vv_t}$ & GD & \checkmark & \xmark  & \xmark & \xmark & \xmark  & $\M_t = \text{Diag}(\alpha_t )  \M_{t-1}   +  \vv_t\vk_t^\top$\\
         PolySketchFor. & $\inner{\M_t \vk^p_t}{\vv_t}$ & GD & \xmark & \xmark  & \checkmark & \xmark & \xmark & $\M_t =  \M_{t-1}   +  \vv_t (\vk_t^\top)^p$\\
         TTT      & $\| \M_t(\vk_t) - \vv_t \|^2_2$ & GD & \xmark & \checkmark & \xmark & \xmark & \xmark & $\M_t = \M_{t-1} -\eta \nabla \ell(\M_{t-1}; \vk_t, \vv_t)$\\
         DeltaNet & $\| \M_t \vk_t - \vv_t \|^2_2$ & GD & \xmark & \xmark & \xmark & \xmark & \xmark & $\M_t =  (\mathbf{I} - \beta_t \vk_t \vk_t^\top) \M_{t-1} + \beta_t \vv_t  \vk_t^\top$\\
         Longhorn & $\| \M_t \vk_t - \vv_t \|^2_2$ & Implicit GD & \xmark & \xmark & \xmark & \xmark & \xmark & $\M_t =  \left(\mathbf{I}- \delta_t \vk_t \vk^\top \right) \M_{t-1} + \left(\delta_t \odot \vv_t\right) \vk_t$ $^\S$\\
         Gated DeltaNet & $\| \M_t \vk_t - \vv_t \|^2_2$ & GD & \checkmark & \xmark & \xmark & \xmark & \xmark & $\M_t =   \alpha_t (\mathbf{I} - \beta_t \vk_t \vk_t^\top) \M_{t-1} + \beta_t \vv_t  \vk_t^\top$\\
         RWKV-7 & $\| \M_t \vk_t - \vv_t \|^2_2$ & GD & \checkmark & \xmark & \xmark & \xmark & \xmark  & $\M_t =    (\text{diag}(\alpha_t) - \beta_t \vk_t \vk_t^\top) \M_{t-1}\! + \beta_t \vv_t  \vk_t^\top$\\
         \vspace{-6pt}\\
         \multirow{2}{*}{Titans} & \multirow{2}{*}{$\| \M_t(\vk_t) - \vv_t \|^2_2$} & \multirow{2}{*}{GD w/ M.$^*$} & \multirow{2}{*}{\checkmark} & \multirow{2}{*}{\checkmark} & \multirow{2}{*}{\xmark} & \multirow{2}{*}{\xmark} & \multirow{2}{*}{\xmark} & $\M_t = \alpha_t \M_{t-1} + \mathcal{S}_t$\\
         & &  & & & & & & $\mathcal{S}_t = \eta_t \mathcal{S}_{t-1} - \theta_t \nabla \ell(\M_{t-1}; \vk_t, \vv_t)$\\
         \vspace{-6pt}\\
         Titans-- & $\| \M_t(\vk_t) - \vv_t \|^2_2$ & GD & \checkmark & \checkmark & \xmark & \xmark & \xmark & $\M_t = \alpha_t \M_{t-1} -\eta_t \nabla \ell(\M_{t-1}; \vk_t, \vv_t)$\\
         \textsc{Moneta} & $\| \M_t(\vk_t) - \vv_t \|^p_p$ & GD & \checkmark & \checkmark & \checkmark & \xmark & \xmark & $\M_t = \alpha_t \M_{t-1} - \eta_t \nabla \ell(\M_{i-1};\vk_t, \vv_t)$ \\
         \textsc{Memora} & $\| \M_t(\vk_t) - \vv_t \|^2_2$ & GD & \checkmark & \checkmark & \xmark & \xmark & \xmark  & $\M_t = \sigma\left( \alpha_t \log(\M_{t-1})  - \eta_t  \nabla \ell(\M_{t-1};\vk_t, \vv_t)\right)$\\
         \midrule
         \multicolumn{9}{c}{\textcolor{black}{\textbf{Our Models}}}\\
         \midrule 
         DLA & $\inner{\M_t(\phi(\vk_t))}{\vv_t}$ & GD & \checkmark & \checkmark & \xmark & \xmark & \xmark &  $\M_t = \alpha_t \M_{t-1} -\eta_t \nabla \ell(\M_{t-1}; \vk_t, \vv_t)$ \\
         DeepTransformer$^\divideontimes$ & $\inner{\M_t(\phi^{*}(\vk_t))}{\vv_t}$ & GD & \checkmark & \checkmark & \checkmark & \xmark & \xmark &  $\M_t = \alpha_t \M_{t-1} -\eta_t \nabla \ell(\M_{t-1}; \vk_t, \vv_t)$ \\
         SWDT & $\sum_{i = c}^{L} \inner{\M_t(\phi^{*}(\vk_i))}{\vv_i}$ & GD & \checkmark & \checkmark & \checkmark & \xmark & \checkmark & $\M_t = \alpha_t \M_{t-1} -\eta_t \nabla \ell(\M_{t-1}; \vk_t, \vv_t)$ \\
         \learningrule Net & $\sum_{i = c}^{L} \gamma_i \left\| \M_t(\phi(\vk_i)) - \vv_i \right\|^2_2$ & GD & \checkmark & \checkmark & \checkmark & \xmark & \checkmark & $\M_t = \alpha_t \M_{t-1} -\eta_t \nabla \ell(\M_{t-1}; \vk_t, \vv_t)$\\ 
         \textsc{Dot}$^\divideontimes$ & $\sum_{i = c}^{L} \gamma_i \left\| \M_t(\phi^{*}(\vk_i)) - \vv_i \right\|^2_2$ & GD & \checkmark & \checkmark & \checkmark & \xmark & \checkmark & $\M_t = \alpha_t \M_{t-1} -\eta_t \nabla \ell(\M_{t-1}; \vk_t, \vv_t)$\\ 
          \vspace{-6pt}\\
         \multirow{2}{*}{\model} & \multirow{2}{*}{$\sum_{i = c}^{L} \gamma_i \left\| \M_t(\phi(\vk_i)) - \vv_i \right\|^2_2$} & \multirow{2}{*}{Muon} & \multirow{2}{*}{\checkmark} & \multirow{2}{*}{\checkmark} & \multirow{2}{*}{\checkmark} & \multirow{2}{*}{\checkmark} & \multirow{2}{*}{\checkmark} & $\M_t = \alpha_t \M_{t-1} - \eta_t \:\: \texttt{NS-5}(\SSS_t)$ \\
         & & & & & & & & $\SSS_t = \theta_t \SSS_{t-1} - \nabla \ell(\M_{t-1}; \vk_t, \vv_t)$ \\
    \toprule 
     \multicolumn{9}{l}{$^\dagger$ The matrix-valued memory version is considered. \qquad $^\ddagger$ NP: Nonparametric \qquad $^\S$ $\delta_t = \frac{\beta_t}{\mathbf{1}+ \beta_t \vk_t^\top \vk_t}$ . \qquad $^*$ Gradient Descent with Momentum. \qquad $^\divideontimes$ Without Normalization.}\\
    \end{tabular}
    }
    \label{tab:my_label}
\end{table*}

\subsection*{Memory Perspective}
Associative memory—the ability to map different entities or events—is an inseparable component of learning in humans~\citep{terry2017learning} and so has motivated several recent studies to understand the state-of-the-art deep learning architectures through its lens~\citep{ramsauer2021hopfield, behrouz2024titans, behrouz2025Miras, wang2025test}. In this perspective, memory is defined as a neural update caused by an input; the more surprising the input is, the more it affects the memory and so is memorable. Therefore, finding an effective ``surprise metric'' is a critical step towards designing such memory modules. As earlier discussed by \citet{behrouz2025Miras, behrouz2024titans}, almost all existing architectures use a surprise metric that updates the memory based on the current input. An event (as a sequence of tokens), however, might not consistently be surprising through a long-period of time although it is memorable. To overcome this issue, \citet{behrouz2024titans} suggest breaking the surprise metric into two parts of ``momentary'' and ``past'' surprise, incorporating the \emph{cumulative} surprise of past inputs when updating the memory with respect to the current input. This design, however, can miss the \emph{context} by memorizing individual tokens. To this end, in this work, we present a long-term neural memory module that measures the surprise of a local (or global) context window, meaning that it learns how to memorize the (\st{token}) context at test time.

Through the paper, we use the terminology ``Test Time Memorization'' because the process involves storing and retrieving information strictly within the global context, without updating the model's core learned parameters (i.e., outer-loop) or initial states from pre-training. Typically, no persistent learning or skill acquisition carries over to new, independent global context once the memory is cleared. Thus, we prefer the use of "test time memorization" over~using~"test~time~training".

\subsection*{Contributions}
In this paper, we aim to overcome the abovementioned limitations—i.e., (1) online nature, (2) limited memory capacity, and (3) less expressive memory management—by designing a long-term neural memory module with high capacity and the ability to memorize the context, instead of tokens. We further build upon these insights and present a family of strictly more powerful Transformers. More specifically:

\head{Better Understanding of Memory Capacity and its Bottleneck}
To improve the limited memory capacity, we suggest using higher-order feature mappings (e.g., polynomial feature kernels) on input tokens. We provide theoretical justifications on why deeper memory modules and/or higher-order feature mapping can enhance memory capacity—i.e., the maximum number of linearly independent key-value associations the memory can perfectly map.

\head{New Expressive Learning Rule}
To overcome the online nature of recent recurrent models, this work presents a sliding window update rule, called \learningrule{} rule, that optimizes and updates memory based on all past tokens in a given context window, not just the last. This allows the model to better manage its fixed-size memory and memorize a local context instead of individual tokens. 

\head{Strict Generalization of Transformers} Next, we show how our \learningrule{} rule formulation connects to global and local softmax attentions (i.e., Sliding Window Attention - SWA) and present a new family of Transformer-like architectures, called \textsc{DeepTransformers} and its sliding window variants SWDT, that strictly generalize Transformers~\citep{transformers}. We further present a novel baseline of Deep Linear Attention (DLA) to demonstrate the role of deep memory.

\head{New Memory Modules with Better Memory Management} 
Building upon the above improvements, we present \omodel{}, a new architecture using polynomial features on its keys and queries, while updating its memory based on  \learningrule{} and gradient descent. To further enhance memory management, we introduce \model{}, which leverages the popular Muon optimizer~\citep{jordan2024muon} for updating the internal memory. We show that both \omodel~and \model{} can take advantage of parallelizable training algorithms, resulting in fast training without substantial overhead compared to the online version (i.e., context window = 1). To the best of our knowledge, \model{} is the first parallelizable recurrent architecture that optimizes the memory using the (approximation) of second-order information (i.e., has locally optimal memory module).

\head{Improvement on Diverse Downstream Tasks} 
Extensive experiments validate our model designs and proposed techniques, including ablations of modern architectures. We evaluated \textsc{DeepTransformer}s, \omodel{}, and \model{} on diverse benchmarks—language modeling, common-sense reasoning, recall-intensive, and needle-in-haystack tasks—where they outperformed modern linear RNNs, local attention (SWA), and Transformers. Furthermore, we studied the effects of memory architecture, feature mapping, memory management algorithm (internal optimizer), and  \learningrule{} rule on memory module capacity and performance in long-context understanding tasks.

Proofs, additional experimental results, discussions on related work, and the details of experiments are in Appendix.
\section{Preliminaries}\label{sec:prelim}

In this section, we first discuss the notation that we use through the paper and then review the background concepts and related work. 
Additional discussion on related studies are in \autoref{app:rw}. 

\head{Notations} We let $x \in \R^{N \times d_{\text{in}}}$ be the input, $\M_t$ be the state of memory $\M$ at time $t$, $\mathbf{K}$ be the keys, $\mathbf{V}$ be the values, and $\mathbf{Q}$ be the query matrices. We use bold lowercase letters with subscript $t$ to refer to vectors correspond to time $t$ (i.e., $\vk_t, \vv_t$, and $\vq_t$). Following  \citet{behrouz2025Miras}, we use $\ell(\M_t; \vk_t, \vv_t)$ to refer to the attentional bias (i.e., the internal memory objective). Through the paper, we use simple MLPs with $\mathcal{L}_{\M} \geq 1$ layers and residual connection as the architecture of the memory module $\M(\cdot)$. Notably, despite this choice, all of our model formulations are simply adaptable to other memory architecture choices; e.g., linear matrix-valued memory ($\mathcal{L}_{\M} = 1$). When it is needed, we parameterized the memory module with $\boldsymbol{\theta}_{\M} := \{W_1, \dots, W_{\mathcal{L}_{\M}}, \dots \}$, which at least includes the parameters of linear layers in the MLP.

\subsection{Backgrounds}

\head{Attention}
Attention is a critical component of Transformers that acts as their associative memory~\citep{bietti2023birth, sun2024learning, behrouz2025Miras}. Given input $x \in \R^{N \times d_{\text{in}}}$, causal attention computes output $\mathbf{y} \in \R^{N \times d_{\text{in}}}$ over input dependent key, value, and query matrices $\mathbf{Q} = x \mathbf{W}_{\mathbf{Q}}, \mathbf{K} = x \mathbf{W}_{\mathbf{K}}, \: \text{and}\:\: \mathbf{V} = x \mathbf{W}_{\mathbf{V}}$ as:
\begin{align}\label{eq:attention}
    &\mathbf{y}_i = \sum_{j = 1}^{i} \frac{ \exp\left( \mathbf{q}_i^{\top} \mathbf{k}_j/\sqrt{d_{\text{in}}}\right) \mathbf{v}_j }{\sum_{\ell = 1}^{i} \exp\left( \mathbf{q}_i^{\top} \mathbf{k}_{\ell}/\sqrt{d_{\text{in}}}\right)} =  \frac{1}{Z_i} \sum_{j = 1}^{i} \exp\left( \mathbf{q}_i^{\top} \mathbf{k}_j/\sqrt{d_{\text{in}}}\right) \mathbf{v}_j,
\end{align}
where $\mathbf{W}_{\mathbf{Q}}, \mathbf{W}_{\mathbf{K}},$ and $\mathbf{W}_{\mathbf{V}} \in \R^{d_{\text{in}} \times d_{\text{in}}}$ are learnable parameters, and $Z_i = {\sum_{\ell = 1}^{i} \exp\left( \mathbf{q}_i^{\top} \mathbf{k}_{\ell}/\sqrt{d_{\text{in}}}\right)}$ is the normalization term. Despite Transformers' simple parallelizable training and effectiveness in recall-intensive tasks~\citep{arora2024simple}, their generation process and long-context scaling are significant drawbacks, as attention requires at least
$N\times d$ operations per token to calculate the output (see \autoref{eq:attention}). Therefore, in recent years, there have been an extensive research effort to design alternative architectures. We divide and review these studies into two groups: (1) Linear shallow memory recurrent models, (2) Deep memory modules:

\head{(Linear) Recurrent Models}
Linear RNNs have recently gained attention as efficient Transformer alternatives due to their parallelizable, linear-time training and comparable performance~\citep{sun2023retentive, peng2023rwkv}. Early modern RNN variants, often based on Hebbian~\citep{hebb2005organization} or Delta~\citep{widrow1988adaptive} learning rules, compress data into vector-valued or matrix-valued memory~\citep{katharopoulos2020transformers, sun2023retentive, polysketchformer-kacham2024, liu2024longhorn, schlag2021linear, lim2024parallelizing}.
Let $\M_t \in \mathbb{R}^{d \times n}$ be the memory (where $n=1$ yields vector-valued memory), and $\vk, \vv \in \mathbb{R}^{d}$ be the keys and values (projections of input $x_t \in \mathbb{R}^{d}$)). A simple general formulation for such linear RNNs is:
\begin{align}
    \M_t = A_{t} \ast \M_{t-1} + \vv_t \vk_t^{\top}, 
\end{align}
where $\ast$ is an arbitrary associative operator and $A_t$
  is a data-(in)dependent diagonal or low-rank plus identity matrix~\citep{yang2024parallelizing}. Despite the efficient \emph{linear} recurrent nature of these models, their memory can overflow, particularly with increasing context length. Although forget gates have recently significantly improved memory management in these architectures~\citep{sun2023retentive, peng2025rwkv7}, their memory's expressivity remains bounded by its linear structure.

\head{Deep Memory Module}
To overcome the limited expressivity of memory and to enhance the \emph{effective} context length recurrent models, recent studies focus on a new line of architectures with deep memory modules~\citep{irie2021going, sun2024learning, behrouz2024titans, behrouz2025Miras}. These architectures are built on the meta-learning perspective, where the memory is a deep MLP architecture updated by gradient descent (with momentum). Recently, \citet{behrouz2025Miras} present a framework to accurately unifies popular sequence models as the instances of test time memorization. That is, sequence models are associative memory modules that aim to learn the underlying mapping between given keys and values by optimizing an internal memory objective, called attentional bias. This optimization is based on an iterative optimization algorithms such as gradient descent. More formally, associative memory is defined as:
\begin{dfn}[\citet{behrouz2025Miras}] \label{dfn:associative-memory}
Given a set of keys $\mathcal{K} \subseteq \mathbb{R}^{d_k}$ and values $\mathcal{V} \subseteq \mathbb{R}^{d_v}$, associative memory is an mapping $\M: \mathcal{K} \rightarrow \mathcal{V}$. Learning the associative memory is based on an objective $\mathcal{L}$, called \emph{Attentional Bias}, that determines the type of memory and its priorities:
\begin{align}\label{eq:attentional-bias-loss}
    \M^* = \arg\min_{\M}\quad \mathcal{L}(\M(\mathcal{K}); \mathcal{V}).
\end{align}
\end{dfn}
Optimizing this objective using an iterative algorithm (e.g., gradient descent) results in the memory update rule. Thus, the sequence model is a meta in-context learner with two  optimization levels:
\begin{enumerate}
    \item \textcolor{c4}{Inner Loop}: Where parameters of the memory module are optimized (i.e., $\boldsymbol{\theta}_{\M} = \{W_1, W_2, \dots, W_{\mathcal{L}_{\M}, \dots}\}$). In the inner optimization loop, all other parameters from the model are considered hyperparameters and are fixed and \emph{not} optimized.   
    \item \textcolor{c4}{Outer Loop}: Where all other parameters of the model are optimized, such as linear projections, MLP layers, convolutions, etc. 
\end{enumerate}
Our terminology builds on this framework. Therefore, instead of full recurrent formulations, we describe models by their: (1) memory architecture, (2) internal objective (i.e., attentional bias), and (3) memory learning algorithm (optimizer). In most cases, models use matrix-valued memory with online gradient descent; for brevity in such instances, we refer to an architecture solely by its internal memory objective. For additional discussions and examples, see \autoref{app:miras}.

\begin{figure*}
\vspace{-2ex}
    \centering
    \includegraphics[width=\linewidth]{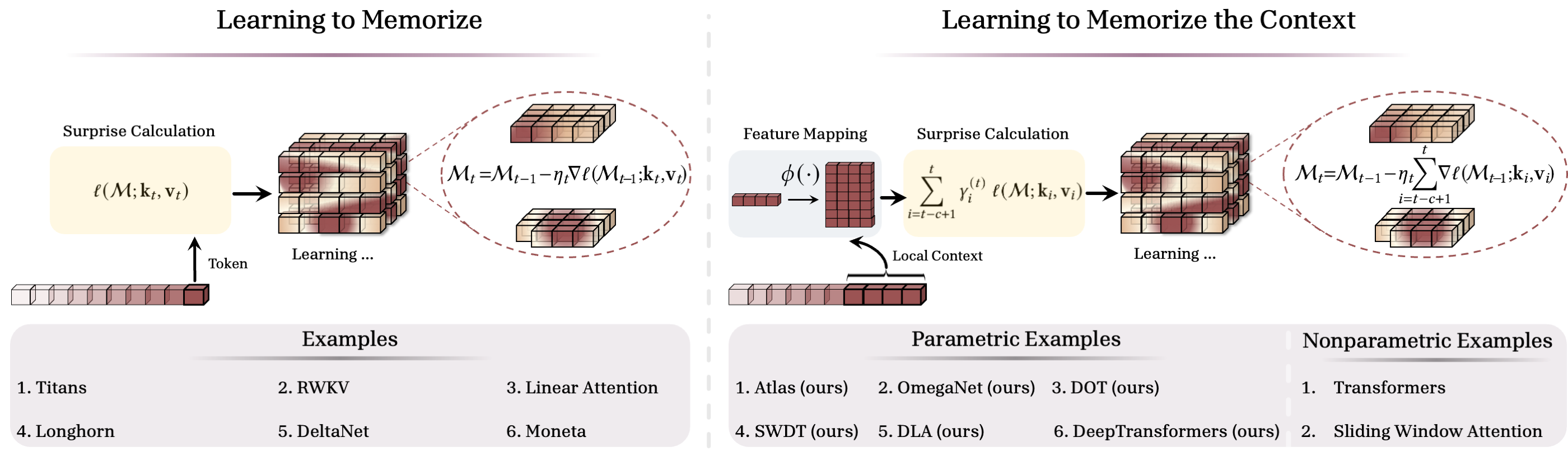}
    \caption{Comparison of learning to memorize (\textbf{Left}) individual tokens, and (\textbf{Right}) the context.}
    \vspace{-2ex}
    \label{fig:omega-rule-examples}
\end{figure*}

\section{Learning to Memorize the Context at Test Time}
Long-term associative memory, crucial for human learning~\citep{terry2017learning}, has inspired many artificial neural architectures~\citep{he2024camelot, krotov2016dense, LSTM, ramsauer2021hopfield, hopfield1982neural, behrouz2024titans, behrouz2025Miras}. While many such models use matrix- or vector-valued memory to compress past data~\citep{von2023uncovering, yang2024gated, schlag2021linear}, recent studies advocate for deep non-linear neural memory that encodes past abstractions into its parameters~\citep{sun2024learning, behrouz2024titans, behrouz2025Miras, dalal2025one}. For long-context reasoning/understanding, however, these long-term neural memory modules still require: (1) High capacity—the maximum (key, value) pairs storable in parameters (see \textcolor{c1}{\S}\ref{sec:capacity}); (2) A powerful internal memory objective (i.e., \emph{attentional bias}) to learn complex mapping between keys and values (see \textcolor{c1}{\S}\ref{sec:omega-rule}); (3) Powerful memory management for better fixed-size memory management (see \textcolor{c1}{\S}\ref{sec:omega-rule}); and (4) An efficient parallel training process for large-scale training on modern accelerators (see~\textcolor{c1}{\S}\ref{sec:parallelization}).

This section further discusses these challenges and presents  \learningrule{} rule: an expressive memory update rule with direct access to tokens in a local context window, which memorizes context rather than individual tokens.

\subsection{Associative Memory with Super Linear Capacity}\label{sec:capacity}
As previously discussed, an effective long-term memory module should store past data abstractions in its parameters. However, with a fixed number of memory parameters, a key unanswered question remains: ``\emph{what is the maximum number of uncorrelated (key, value) pairs that a model can store?}'' To answer this, we start with the simplest case: matrix memory, an $\ell_2$ regression loss as the attentional bias (i.e., $\ell(\M_t; \vk_t, \vv_t) = \|\M_t(\vk_t) - \vv_t \|^2_2$), optimized by gradient descent:

\begin{restatable}[Capacity of $\ell_2$ Attentional Bias]{prop}{Capacity}\label{prop:linear-capacity}
 Let $\M$ be a matrix-valued memory with $d_v \times d_k$ parameters that optimizes the internal objective of $\ell(\M_t; \vk_t, \vv_t) = \|\M_t \vk_t - \vv_t \|^2_2$ with gradient descent. $\M$ can store the mapping of at most $\mathcal{O}\left( d_k \right)$ pairs of $(\vk_i, \vv_i)$ with linearly independent keys. 
\end{restatable} 
The above proposition indicates that matrix-valued memory with delta update rule has sub-linear capacity with respect to its number of parameters. This means that the number of independent patterns that can be stored in a fixed-size memory with size $M$ is strictly less than $c \times M$, for some $c \in \mathbb{R}^{+}$. Recent recurrent models suggest using deep memory modules to store the abstraction of the past into the parameters of a deep neural network~\citep{irie2021going, sun2024learning, behrouz2024titans, behrouz2025Miras}. While these deep memory architectures can intuitively enhance the expressive power in modeling complex underlying mapping patterns between keys and values, it is still unclear that if they enhance the memory capacity.  

\begin{restatable}[Effect of Deep Memory]{theorem}{DeepMemory}\label{thm:deep-memory}
 Let $\M(\cdot)$ be an MLP with $\mathcal{L}_{\M} \geq 2$ layers, $d_k$  input dimension, and $d_h$  hidden dimension. Then, $\M(\cdot)$ can store the mapping of at least $\mathcal{O}\left( d_k d_v \right)$ and at most $\mathcal{O}\left( d_k d_v \sum_{i = 1}^{\mathcal{L}_{\M}} \min \{d_h^{(j)}\}_{j \geq i} d_h^{(j+1)} \right)$ pairs of $(\vk_i, \vv_i)$ with linearly independent keys.  
\end{restatable}

This theorem indicates that deep memory not only improves representational power but also further boosts network capacity, with advantages growing with depth. However, the upper bound remains subquadratic in key and value dimensions, raising the question if a long-term memory module can achieve super-linear capacity.

As stated earlier, the dimension of $\vk_t$s is crucial for increasing memory capacity. Simply increasing all key and value dimensions, however, significantly increase the number of parameters ($\mathcal{O}(d_{\text{in}})$ per each extra dimension) and memory usage, particularly with long contexts. To address this, building on methods from \citet{kacham2024polysketchformer, krotov2016dense}, we suggest using separable kernels $\sigma(x, y) = \phi(x)^{\top} \phi(y)$ for keys and queries. As an example of such kernels, we focus on polynomial kernels of degree at most $p$ to increase input dimensionality and thus network capacity. Given $p \in \mathbb{N}$,  let $\phi_{p}(x) = [x^{\beta}]_{|\beta| \leq p}$ be a polynomial mapping of $x$ with degree at most $p$. We redefine the associative memory module in Definition~\ref{dfn:associative-memory} by replacing the inner objective of $\mathcal{L}(\M(\mathcal{K}); \mathcal{V})$ with $\mathcal{L}(\M\left(\phi\left(\mathcal{K}\right)\right); \mathcal{V})$. This polynomial mapping enhances representational power by increasing the effective dimensionality of keys without additional parameter overhead for the input projections. Next, we discuss their effect on memory capacity, even with a single~matrix-valued~memory:

\begin{restatable}[Memory Capacity with Polynomial Mapping]{prop}{PolyCapacity}\label{thm:polynomial-capacity}
 Let $\phi_p(\cdot)$ be a polynomial mapping with degree at most $p$, and $\M$ be a matrix-valued memory that optimizes the internal objective of $\ell(\M_t; \phi_p(\vk_t), \vv_t) = \|\M_t \phi_p(\vk_t) - \vv_t \|^2_2$ with gradient descent. $\M$ can store the mapping of at most $\mathcal{O}\left( {d_k}^{p} \right)$ pairs of $(\vk_i, \vv_i)$ with linearly independent keys, where $d_k$ is the dimension of keys $\vk_i$. 
\end{restatable}
Beyond the above intuition, polynomial kernels are further motivated by two perspectives: (1) Approximating Softmax using Taylor series; and (2) Input feature gating. For the sake of clarity, we continue with linear memory and two popular attentional biases i.e., $\ell^{(1)}(\M_t; \vk_t, \vv_t) = \inner{\M_t \vk_t}{\vv_t}$ and  $\ell^{(2)}(\M_t; \vk_t, \vv_t) = \| \M_t \phi(\vk_t) - \vv_t \|^2_2$.  The same process can be applied on other attentional objectives and deep memory modules. Optimizing these objectives using gradient descent in the inner loop results in the following recurrent formulas:
\begin{align} \tag{Hebbian Rule} \label{eq:heb}
    &\ell^{(1)}(\M_t; \vk_t, \vv_t)\:\: : \: \:\M_t = \M_{t-1} + \eta_t \vv_t \phi(\vk_t)^\top, \\ \tag{Delta Rule} \label{eq:delta}
    &\ell^{(2)}(\M_t; \vk_t, \vv_t)\:\: : \:\: \M_t = \left( \mathbf{I} - \eta_t \phi(\vk_t) \phi(\vk_t)^{\top}\right) \M_{t-1} + \eta_t \vv_t \phi(\vk_t)^{\top}.
\end{align}

\head{Kernel Attention Perspective for the Special Case of Hebbian Rule} 
The formulation for (\ref{eq:heb}) is equivalent to kernel linear attentions~\citep{polysketchformer-kacham2024, wang2025test, hua2022transformer, kasai2021finetuning, katharopoulos2020transformers, arora2024simple}. In this viewpoint, the role of $\phi(.)$ is to approximate \texttt{Softmax} or more accurately the exponential kernel. Since exponential kernel with normalization (i.e., \texttt{Softmax}) is not separable, it results in Transformers' quadratic time and memory complexity. However, Transformers' exponential feature map kernel ($\exp(\vq_i^{\top} \vk_j)$) can be approximated using its Taylor series as:
\begin{align}
    \exp\left(\vq_i^{\top} \vk_j\right) \approx 1 + \vq_i^{\top} \vk_j + \frac{(\vq_i^{\top} \vk_j)^2}{2!} + \frac{(\vq_i^{\top} \vk_j)^3}{3!} + \dots
\end{align}
Our polynomial feature map extends this approximation to a more general case of:
\begin{align}
    \exp\left(\vq_i^{\top} \vk_j\right) \approx \phi_p(\vq) \phi(\vk)^{\top} \!= a_0 + a_1 \vq_i \vk_j^{\top} + a_2 (\vq_i^{\top} \vk_j)^2 + a_3 (\vq_i^{\top} \vk_j)^3 + \dots + a_p (\vq_i^{\top} \vk_j)^p,
\end{align}
with learnable parameters $a_i \in \R$ initialized at $a_i = \frac{1}{i !}$, the polynomial kernel can be viewed as an expressive approximator of \texttt{Softmax} attention. This provides theoretical motivation for using polynomial kernels, especially when memory capacity is limited; i.e., with (i) linear memory and (ii) Hebbian learning rule. This intuition, however, further generalizes to more expressive cases using deep memory modules and more complex attentional biases (i.e., Eq. \ref{eq:delta}). That is, $\exp(\cdot)$ feature mapping
has infinite dimension and provides a more powerful similarity measure of keys and queries (i.e., $\vq_i^{\top} \vk_j$); however, its computation with normalization can cause additional memory and time complexity to the model. Using polynomial kernels in architectures with deep memory and complex attentional bias can further enhance  performance by approximating  more powerful representations for keys-queries similarities (i.e., $\vq_i^{\top} \vk_j$). See \autoref{sec:deeptransformers} for additional discussions on exponential kernels and Transformers.

\head{Input Gating Interpretation} Another perspective that motivates the use of polynomial features is their more expressive representational power in modeling complex functions compared to the simple case of $\phi(x) = x$. That is, the coefficients of $a_i$s can be seen as input feature gating, in which $a_i \rightarrow 0$ means excluding the feature map of $[x^j]_{|j| = i}$, and $a_i \rightarrow 1$ means retaining the corresponding feature. This is similar to the gating mechanisms of RNNs but on the input rather than the memory. This gating mechanism clearly provides a more representational power as the model can learn to set $a_i \rightarrow 0$ for all $i \neq 1$ and $a_1 \rightarrow 1$, resulting in the simple case of $\phi(x) = x$.

\subsection{Long-term Memory with Context Memorization}\label{sec:omega-rule}
As discussed earlier, one of the critical drawback of most existing recurrent models is their online nature, in which they optimize the inner objective (attentional bias) with respect to only the current input while retaining the previous state of the memory~\citep{behrouz2025Miras, liu2024longhorn}, i.e.,
\begin{align}
    \min_{\M} \ell(\M; \vk_t, \vv_t) + \texttt{Ret}_t(\M, \M_{t-1}),
\end{align}
where $\texttt{Ret}(\cdot, \cdot)$ is the retention gate. This online nature while making the optimization of the memory simpler and faster, can cause sub-optimal memorization of the context as memory is greedily memorize individual tokens. In a more general case, however, one can optimize the memory at each time stamp with respect to the entire context (input sequence), i.e.,
\begin{align}\label{eq:global-optimization}
     \min_{\M} \sum_{i = 1}^{t} \ell(\M; \vk_i; \vv_i).
\end{align}
This strict global formulation generally presents two critical limitations: (1) Efficiency: One of the important advantages of recurrent architectures is their efficiency at longer context in both training and inference. Optimizing the memory with respect to all the past tokens (entire context), however, (i) causes additional optimization constraints at each memory update step, resulting in inefficiency at extremely large sequences, and (ii) requires caching the past keys and values at the test time, increasing the memory consumption; (2) Context Pruning: In large context tasks optimizing with all past tokens can cause sub-optimal performance mainly due to the context change (or irrelevant context) in the middle of the input sequence. This observation has resulted to design architectures with retention (forget) gate, enabling models to erase memory when past context is no longer needed~\citep{sun2023retentive, peng2025rwkv7, behrouz2024titans, behrouz2025Miras, yang2024gatedattn}.

To address these limitations, we present a sliding window recurrent model that optimizes its attentional bias w.r.t. a window of past tokens. For a memory module $\M(\cdot)$ and window length~$c \geq 1$, we optimize the memory internal objective as:
\begin{align}
    \min_{\M} \sum_{i = t - c + 1}^{t} \gamma^{(t)}_i \: \ell(\M; \vk_i, \vv_i), 
\end{align}
where $\ell(\M; \vk_i, \vv_i)$  measures the predicted mapping for $(\vk_i, \vv_i)$ pair and $\gamma^{(t)}_i$ is the decay term for the effect of $i$-th token in the optimization process. Building upon this formulation, 
we present  \learningrule{} rule, which is strictly more powerful than the popular Delta learning rule~\citep{widrow1988adaptive, schlag2021linear}:

\begin{c2box}
\noindent
\textbf{\learningrule{} Rule}: Let $\vk_i \in \R^{d_k}$ and $\vv_i \in \R^{d_v}$ be the input keys and values, and $\M(\cdot)$ be a neural architecture that serves as the memory module. Given a local context length of $c \in \mathbb{N}_{\geq 1}$, the updating the memory module $\M$ using \learningrule{} learning rule is defined as optimizing the following loss function with gradient descent: 
\begin{align}\label{eq:omega}
    \min_{\M} \sum_{i = t - c + 1}^{t} \gamma^{(t)}_i \left\| \M\left(\vk_i\right) - \vv_i \right\|^2_2
\end{align}
\end{c2box}
Following \cite{behrouz2025Miras}, this update rule can  be extended to $q$-\learningrule{} rule (or other variants) by replacing $\ell_2(\cdot)$ with $\ell_q(\cdot)$. 
In the extreme cases of (1) $c =1$: the update rule becomes online (Delta  rule); and (2) $c = \infty$ or context length: the update becomes global optimization w.r.t. all past tokens. In this formulation, parameters $\gamma^{(t)}_i \in [0, 1]$ act as hard (direct) gates for the past tokens. That is, $\gamma^{(t)}_i \rightarrow 0$ means that the model directly prunes the optimization of $i$-th token in the local context, while $\gamma^{(t)}_i \rightarrow 1$ means fully incorporating the optimization of memory for $i$-th token in the local context. In our design, we use input-dependent parameters for $\gamma^{(t)}_i$, providing in-context pruning ability. Note that, the design of sliding window recurrence allows such flexibility as for each token we need a constant number of gates; i.e., $\{\gamma^{(t)}_i\}_{i = 1}^{c}$. Using input-dependent gates for the global optimization (\autoref{eq:global-optimization}), however, can result in significant parameter increase and memory usage, diminishing the advantages of recurrent models.

\head{\omodel} We now present \omodel, a novel sequence model that updates its memory using \learningrule{} rule. To enhance the memory capacity of \omodel, we use polynomial kernels on $\vk$s and $\vq$s. Accordingly, optimizing the objective in \autoref{eq:omega}, results in an update rule of \omodel{} as:
\begin{align}
    \M_t = \alpha_t \M_{t-1} - \undermath{\textcolor{c4}{\text{Surprise of the context}}}{\nabla \sum_{i = t - c + 1}^{t} \gamma^{(t)}_i \left\| \M\left(\phi(\vk_i)\right) - \vv_i \right\|^2_2}, 
\end{align}
or in the spacial case of linear memory:
\begin{align}
    \M_t = \left( \texttt{diag}(\alpha_t) - \sum_{i = t-c+1}^{t} \gamma^{(t)}_i \phi(\vk_i)\phi(\vk_i)^{\top} \right)  \M_{t-1} - \sum_{i = t-c+1}^{t} \gamma^{(t)}_i \vv_i\phi(\vk_i)^{\top}. 
\end{align}
From the memory perspective, \learningrule{} rule (\omodel) does not measure the surprise of a token, but the surprise of a local context based on the context-aware combination of individual tokens within the context.  

\begin{figure*}
    \centering
    \includegraphics[width=0.9\linewidth]{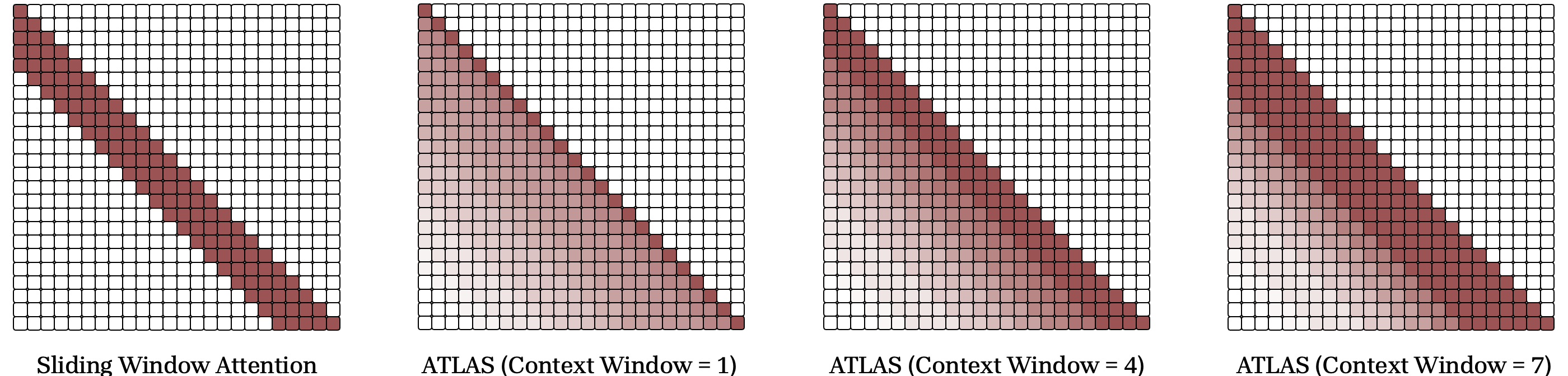}
    \caption{The illustration of tokens dependencies in SWA and \model{} or \omodel{} with different context length.}
    \label{fig:attn_ATLAS}
\end{figure*}

\head{Beyond Gradient Descent} The concept of \learningrule{} rule and ``test time memorization of context'' can simply be extended to optimizing the objective in \autoref{eq:omega} with any arbitrary optimizer, even beyond simple gradient descent. We use two extreme cases for $c$ as the illustrations. In the first case, we let $c = 1$, $\gamma^{(t)}_i = 1$, and use gradient descent with momentum as the optimizers, resulting in the following update rule:
\begin{align}
    &\M_t = \alpha_t \M_{t-1} + \SSS_t \\
    &\SSS_t = \theta_t \SSS_{t-1} - \eta_t \nabla \ell(\M_{t-1}; \vk_t, \vv_t).
\end{align}
This update rule is equivalent to the long-term neural memory in Titans~\citep{behrouz2024titans}. In the second case, using a linear memory $\M$, letting $\gamma^{(t)}_i = 1$, and $c$ be equal to the context length, the memory update process is equivalent to optimizing the (regularized) least-squares problem:
\begin{align}
    \M_t = \min_{\M} \sum_{i = 1}^{t} \left\| \M \vk_i - \vv_i \right\|^2_2.
\end{align}
\citet{von2023uncovering} suggest directly optimizing the above objective and use Sherman-Morrison formula~\citep{sherman1950adjustment} to recursively calculate the inverse term in the optimal solution. Despite the optimality of memory, such direct solutions comes with the cost of non-parallelizable training and also are limited to only the linear matrix-valued memory setup. Furthermore, as discussed earlier, the global nature without any direct hard gating terms (i.e., $\gamma^{(t)}_i$s) can force the model to \emph{not} prune the context, damaging the performance in longer sequences.

\subsection{Parallelizing \learningrule{} Rule}\label{sec:parallelization}
While \learningrule{} rule provides a more general and expressive formulation for the design of memory modules than Hebbian or Delta learning rules, its applicability to large-scale models relies on its efficiency in training. To this end, we discuss a fast parallelizable training algorithms that does not add any significant computational overhead with the online counterpart version (i.e., $c = 1$). A  naive implementation requires materializing \( c \) gradients \( \nabla \ell \in \mathbb{R}^{d_{\text{in}} \times d_{\text{in}}} \), which can result in a significantly higher memory footprint and I/O cost when \( d_{\text{in}} \) is large. Also, to fully utilize hardware accelerators such as TPUs and GPUs, it is important to tensorize computations and maximize the use of \texttt{matmul} operations. Motivated by recent work~\citep{behrouz2024titans, sun2024learning}, we propose a simple sliding window masking strategy that supports efficient parallel training while avoiding substantial memory overhead. Specifically, we partition the input sequence with length $L$ into chunks of size \( b  \geq 1 \), each of which is represented by $\mathbf{S}_i = \{ \vx_{(i-1)b+1}, \dots, \vx_{i b}\}$. Then for each chunk, we calculate the gradients with respect to the last state of the previous chunk. For the sake of clarity, we first assume \( \gamma^{(t)}_i =  \eta_t \) for all positions in the sequence. When the chunk size is \( b = 1 \), the update rule is:
\begin{align}
    \M_t = \alpha_t \M_{t-1} - \eta_t \sum_{i=t-c+1}^{t} \nabla \ell(\M_{t-1}; \vk_i, \vv_i),
\end{align}
where \( \M_t \) is the model state at step \( t \), \( \alpha_t \) and \( \eta_t \) are the weight decay and learning rate parameters respectively, and \( (\vk_i, \vv_i) \) denote the input pair at position \( i \). In practice, we strike a balance between the fully recurrent form and the fully parallel form by dividing the sequence into smaller chunks. Within each chunk (intra-chunk), we apply parallel computation, while across chunks (inter-chunk), we adopt a recurrent computation scheme. We now define \( t' = t - \operatorname{mod}(t, b) \). That is, for time steps \( t \) such that \( t' \leq t < t' + b \), the update rule within each chunk becomes:
\begin{align} \label{eq:omega_chunk_update_rule}
    &\M_t = \alpha_t ... \alpha_{t'} \M_{t'} - \sum^{t}_{n=t'} \frac{\alpha_t ... \alpha_{t'}}{\alpha_n ... \alpha_{t'}}\eta_n \undermath{G_t}{\sum^{n}_{i=n-c+1}\nabla \ell(\M_{t'}; \vk_i, \vv_i)}
\end{align}
In our implementation, for $G_t$, we follow the same gradient computation approach as described in Titans~\citep{behrouz2024titans} but additionally apply a sliding window mask \( M_s \) during the broadcasting operation (e.g., using \texttt{einsum}). When \( c = 1 \), the sliding window mask \( M_s \) reduces to the identity matrix. For \( c > 1 \), \( M_s \) is an identity matrix except that the \( c - 1 \) positions immediately preceding each diagonal entry are also set to 1. This allows gradient contributions from a window of size \( c \), enabling efficient computation without materializing all gradients inside the chunk.

\section{\protect \textsc{DeepTransformer}s: Transformers with Deep Memory}\label{sec:deeptransformers}
Recent studies have extensively discussed Transformer architectures through the lens of associative memory~\citep{wang2025test, sun2024learning, behrouz2025Miras}. Accordingly, it is natural to ask how our discussions of memory capacity as well as \learningrule{} rule can affect Transformers. In this section, we discuss how our formulation of \learningrule{} rule is connected to Transformers and their sliding window counterparts (i.e., SWA). We then further provide two extensions to Transformers, each of which is a strict generalization of Transformers.

\subsection{Online and Local Context Optimization of Memory}

\head{Connection to Sliding Window Attention}
\texttt{Softmax} attention block can also be reformulated as a non-parametric  solution to the $\ell_2(\cdot)$ regression  with  Nadaraya-Watson estimators~\citep{zhang2022analysis, fan2018local}:
\begin{align}
    \M^* = \arg\min_{\M} \sum_{i = 1}^{L} \mathbf{s}(\vk_i, \vq) \|\vv_i - \M \|^2_2
    = \sum_{i = 1}^{L} \frac{\mathbf{s}(\vk_i, \vq)}{\sum_{j = 1}^{L} \mathbf{s}(\vk_j, \vq)} \vv_i,
\end{align}
where $L$ is the sequence length. While this formulation optimizes the memory $\M$ with respect to the entire sequence length, one can limit the optimization process to the past $c$ tokens, resulting in:
\begin{align}
    \M^* = \arg\min_{\M} \sum_{i = t - c + 1}^{t} \mathbf{s}(\vk_i, \vq_i) \|\vv_i - \M \|^2_2  = \sum_{i = t-c + 1}^{t} \frac{\mathbf{s}(\vk_i, \vq)}{\sum_{j = t - c + 1}^{t} \mathbf{s}(\vk_j, \vq)} \vv_i, 
\end{align}
which is equivalent to the sliding window attention (SWA). This connection provides an important insight on the difference of attention and recurrent models: Not only attention is a non-parametric solution (contrary to the parametric nature of recurrent models), it globally optimizes its internal objective (attentional bias), while most recent modern recurrent models are online learners~\citep{sun2024learning, yang2024gated, behrouz2025Miras, peng2025rwkv7}\footnote{Two of the exceptions are Titans~\citep{behrouz2024titans} and Mesa-layer~\citep{von2023uncovering}, where Mesa-layer optimizes the memory with respect to all past tokens (comes with the cost of slow training), and Titans optimizes the memory with respect to all past tokens but with an implicit decay term (i.e., the result of the momentum) for each past token, maintaining parallelizability.}
. Our formulations of sliding window RNN and \learningrule{} rule fill this gap by optimizing the memory with respect to a context window of past tokens based on parametric methods, effectively memorizing the context instead of individual tokens.

\head{Deep Linear Attention}
As a novel baseline, we present Deep (Gated) Linear Attention (DLA) that replaces a matrix-valued memory in (gated) linear attention~\citep{katharopoulos2020transformers, yang2024gatedattn} with a deep neural network (e.g., $k$-layer MLP). As discussed earlier in (\ref{eq:heb}), using dot product similarity as the internal attentional bias results in linear attention. Thus, leveraging recent deep memory modules \citep{sun2024learning, behrouz2024titans, behrouz2025Miras}, we optimize the memory using gradient descent with dot product attentional bias:
\begin{align}\label{eq:DLA}
    \M_t = \alpha_t \M_{t-1} -  \eta_t  \nabla \ell(\M_{t-1}; \phi(\vk_t), \vv_t),
\end{align}
where $\ell(\M_{t-1}; \phi(\vk_t), \vv_t) = \inner{\M_{t-1}(\phi(\vk_t))}{\vv_t}$ and $\phi(\cdot)$ is a polynomial kernel. The training of DLA can simply be parallelized using the hybrid of linear and non-linear chunk-wise training, the same as \citet{sun2024learning, behrouz2024titans} and our discussion in \autoref{sec:parallelization}.

\head{Sliding Window Linear Attention}
Building upon the above intuition and the connection of our formulation to SWA, we present Sliding Window Linear Attention (SWLA) block. Following the formulation of linear attention in associative memory perspective~\citep{behrouz2025Miras}, we use dot product similarity (i.e., $\ell(\M_t; \vk_i, \vv_i) = \inner{\M_t(\vk_i)}{\vv_i}$) as the attentional bias and optimize the loss function using gradient descent. For the sake of clarity, we use a linear memory here to derive the closed form:
\begin{align}
    \M_t = \alpha_t\M_{t-1} - \eta_t \nabla \sum_{i = t - c + 1}^{t} \ell(\M_{t-1}; \phi(\vk_i), \vv_i) = \M_{t-1} +  \sum_{i = t - c + 1}^{t} \gamma^{(t)}_i \vv_i \phi(\vk_i)^{\top}
\end{align}
In the online case ($c=1$) and $\phi(\cdot) = (\cdot)$, this recurrence is the same as linear attention~\citep{katharopoulos2020transformers}.

\subsection{Memory Capacity and Exponential Kernels}\label{sec:deeptransformer-capacity}
We first recall the formulation of \texttt{softmax} attention in Transformers (i.e., \autoref{eq:attention}):
\begin{align}\label{eq:attention2}
    &\mathbf{y}_i = \frac{1}{{\sum_{\ell = 1}^{i} \exp\left( \mathbf{q}_i^{\top} \mathbf{k}_{\ell}/\sqrt{d_{\text{in}}}\right)}} \sum_{j = 1}^{i} \exp\left( \mathbf{q}_i^{\top} \mathbf{k}_j/\sqrt{d_{\text{in}}}\right) \mathbf{v}_j,
\end{align}
which its $\exp(\cdot)$ kernel is not separable and so cannot be written as a recurrence. Following the discussion in \citet{polysketchformer-kacham2024}, one can see $\exp(\cdot)$ kernel (compared to polynomial kernel $\phi_p(\cdot)$) as a feature map that maps the input into an infinite dimension. That is, we define: 
\begin{align}\label{eq:infi-map}
    \phi^{*}(x) = \begin{pmatrix}
    1 \\
    \frac{x}{\sqrt{1}} \\
    \frac{x^{\otimes 2}}{\sqrt{2!}} \\
    \frac{x^{\otimes 3}}{\sqrt{3!}} \\
    \vdots
    \end{pmatrix}, \qquad\qquad \phi_p(x) = 
    x^{\otimes p},
\end{align}
where $x^{\otimes p} = x \otimes x^{\otimes (p-1)}$ is a ``self-tensoring'' operator with Kronecker product~\citep{polysketchformer-kacham2024} and so:
\begin{align}
    \exp(\vq_t^{\top} \vk_t) = \phi^{*}(\vq_t)^{\top} \phi^{*}(\vk_t).
\end{align}
Based on the above kernel, we can reformulate the attention (see \autoref{eq:attention2}) as: (we remove $1/\sqrt{d_{\text{in}}}$ term for the sake of simplicity)
\begin{align}
    &\mathbf{y}_i = \frac{1}{{\sum_{\ell = 1}^{i} \exp\left( \mathbf{q}_i^{\top} \mathbf{k}_{\ell}/\sqrt{d_{\text{in}}}\right)}} \sum_{j = 1}^{i}  \mathbf{v}_j \phi^{*}(\mathbf{k}_j)^{\top} \phi^{*}(\mathbf{q}_i)  = \frac{1}{{\sum_{\ell = 1}^{i} \exp\left( \mathbf{q}_i^{\top} \mathbf{k}_{\ell}/\sqrt{d_{\text{in}}}\right)}} \left(\sum_{j = 1}^{i}  \phi^{*}(\mathbf{v}_j \mathbf{k}_j)^{\top}  \right) \phi^{*}(\mathbf{q}_i) = \M_i \phi^{*}(\mathbf{q}_i),
\end{align}

This formulation, provides another important insight on the differences of attention and (kernel) recurrent models: \texttt{Softmax} attention as an associative memory has an unbounded memory and so can better memorize larger context into its parameters. Building upon this insight, we present \textsc{DeepTransformers} by replacing polynomial kernel with $\phi^{*}(\cdot)$ kernel in Deep Linear Attention formulation (\autoref{eq:DLA}), resulting in unnormalized formulation of:
\begin{align}\label{eq:DeepTransformers}
    \M_t = \M_{t-1} - \nabla  \inner{\M_{t-1}(\phi^{*}(\vk_t))}{\vv_t}.
\end{align}
In the special case of linear memory, we can derive the closed form for the above formulation as:
\begin{align}
    \M_t = \M_{t-1} - \nabla  \inner{\M_{t-1}\phi^{*}(\vk_t)}{\vv_t} = \M_{t-1} + \vv_t \phi^{*}(\vk_t)^{\top}  = \sum_{i=1}^{t} \vv_i \phi^{*}(\vk_i)^{\top}  \quad \Rightarrow \quad \mathbf{y}_t = \M_t \phi^{*}(\vq_t) = \sum_{i = 1}^{t} \vv_i \exp(\vq_i^{\top} \vk_i ) ,
\end{align}
which matches the output of the unnormalized Transformers. Therefore, \textsc{DeepTransformers} are strict generalizations of Transformers with \texttt{softmax} attention~\citep{transformers}.

\subsection{Deep Omega Transformer (\protect \textsc{Dot}): Transformers with \learningrule{} learning rule} \label{sec:dot}
Our above formulation of \textsc{DeepTransformers} is based on the (\ref{eq:heb}), which is also used in original Transformers. However, as discussed earlier, using more powerful memory management and learning rules in associative memory modules can further enhance their performance. To this end, we extend the above formulation by replacing the Hebbian rule with our \learningrule{} learning rule, resulting in an unnormalized formulation of Deep Omega Transformers (\textsc{Dot}):
\begin{align}
     \M_t = \M_{t-1} - {\nabla \sum_{i = t - c + 1}^{t} \gamma^{(t)}_i \left\| \M\left(\phi^{*}(\vk_i)\right) - \vv_i \right\|^2_2}.
\end{align}
We now discuss special instances of \textsc{Dot} to provide further intuition on its generalized formulation. 

\head{Linear Memory} This setup results in the following unnormalized formulation:
\begin{align}
     &\M_t = \left( \mathbf{I} - \sum_{i = t-c+1}^{t} \gamma^{(t)}_i \phi^{*}(\vk_i)\phi^{*}(\vk_i)^{\top} \right)  \M_{t-1} - \sum_{i = t-c+1}^{t} \gamma^{(t)}_i \vv_i\phi^{*}(\vk_i)^{\top} \\ 
     \Rightarrow \:\:&\mathbf{y}_t = \M_t \phi^{*}(\vq_t) = \left( \mathbf{I} - \sum_{i = t-c+1}^{t} \gamma^{(t)}_i \phi^{*}(\vk_i)\phi^{*}(\vk_i)^{\top} \right)  \M_{t-1} \phi^{*}(\vq_t) \ - \sum_{i = t-c+1}^{t} \gamma^{(t)}_i \vv_i\exp(\vq_t^{\top}\vk_i). 
\end{align}
\head{Online Case with $c=1$} We now let $c = 1$:
\begin{align}
    &\M_t = \left( \mathbf{I} - \eta_t \phi^{*}(\vk_t)\phi^{*}(\vk_t)^{\top} \right)  \M_{t-1} - \eta_t \vv_t \phi^{*}(\vk_t)^{\top} \\ 
    \Rightarrow \:\: &\mathbf{y}_t = \M_t \phi^{*}(\vq_t) = \left( \mathbf{I} - \eta_t \phi^{*}(\vk_t)\exp(\vq_t^{\top}\vk_t) \right)  \M_{t-1} \ - \eta_t \vv_t\exp(\vq_t^{\top}\vk_t). 
\end{align}
The above (unnormalized) formulation can be seen as the generalization of Transformers with Delta Rule. Therefore, due to the unbounded memory, \textsc{Dot} not only appends the new keys and values (similar to original Transformers), but it also replaces the new value with its predicted value from the previous state.

\section{\model: A Locally Optimal Memory with High Capacity} \label{sec:model}
Although the design of \learningrule{} rule allows the model to memorize the context instead of individual tokens and also the use of polynomial (or exponential) feature mapping increases memory capacity, the memory management (i.e., optimization of mappings between keys and values) is still limited to a simple gradient descent. This choice of optimizer can lead the model to a low-quality solution at a local optima, damaging the performance of the model in longer contexts. To overcome this issue, we suggest using Muon optimizer~\citep{jordan2024muon} (with weight decay) that not only approximates second-order information, but it also mostly leverages matrix multiplication and can be parallelized across the sequence. Accordingly, the use of Muon for optimizing the internal objective in \autoref{eq:omega}, results in the following update rule:
\begin{align}
     &\M_t = \alpha_t \M_{t-1} - \eta_t \: \texttt{NewtonShulz-$k$}(\mathcal{S}_t),\\
     &\mathcal{S}_t = \theta_t \SSS_{t-1} + \nabla \sum_{i = t - c + 1}^{t} \gamma^{(t)}_i \left\| \M\left(\phi^{*}(\vk_i)\right) - \vv_i \right\|^2_2,
\end{align}
where $c$ is the local context length and $k$ is the number steps for \texttt{NewtonShulz} operations. For the additional discussion on the algorithm and this operation we refer the reader to \citet{jordan2024muon}. Following the literature on Muon optimizer, we know that when $k \rightarrow \infty$, then  $\texttt{NewtonShulz-$k$}(\mathcal{S}_t)$ converges to the nearest semi-orthogonal matrix to the momentum term $\SSS_t$ and so approximate second-order information with a lower error. Therefore, interestingly, parameter $k$ can be considered as an internal test-time compute parameter in \model, where using more steps can potentially result in better memorization.

\subsection{Parallel Training}
In this section, we discussed how the training process of \model{} can be parallelized. For the sake of clarity, we assume $c=1$. Generalizing the process to arbitrary value for $c$ follows the procedure in \autoref{sec:parallelization}. We use the same process as we discussed in \autoref{sec:parallelization} and so chunk the sequence and compute all the gradients with respect to the last state of the previous chunk. Accordingly, using the recurrence of \model{} with momentum but without , we have:
\begin{align}
    &\M_t = \alpha_t \M_{t-1} + \SSS_t \\
    &\SSS_t = \theta_t \SSS_{t-1} - \eta_t \nabla \ell(\M_{t'}; \vk_t, \vv_t).
\end{align}

Since $t'$ is the last state of the previous chunk, we can calculate all the gradients before hand and so we let $u_t = \nabla \ell(\M_{t'}; \vk_t, \vv_t)$. Therefore, we have:
\begin{align}
    &\M_t = \alpha_t \M_{t-1} + \SSS_t \\
    &\SSS_t = \theta_t \SSS_{t-1} - \eta_t u_t.
\end{align}
Now by expanding the second recurrence, we have:
\textbf{\begin{align}
    \SSS_t &= \theta_t \SSS_{t-1} - \eta_t \undermath{u_t}{\nabla \ell(\M_{t'}; \vk_t, \vv_t)}, \\
    \Rightarrow \SSS_t &= \undermath{\beta_t}{\theta_t ... \theta_1} \SSS_0 - \sum_{i = 1}^{t} \frac{\theta_t ... \theta_1}{\theta_i ... \theta_1} \eta_i u_i = \beta_t \SSS_0 - \Theta \odot E \odot G, 
\end{align}}
where $G$ is the gradient matrix, $E$ and $\Theta$ are diagonal matrices with value $\theta$ and $\eta$, and $\odot$ is broadcasting. 

The main advantage of the above formulation (chunk wise recurrence) is that the recurrence of momentum is independent of the state of memory. That is, we can calculate all the momentum terms in the beginning of the chunk using the above formulation. Now in the Muon case, we want to use Newton-Schulz algorithm on the momentum terms, which results in:
\begin{align}
    &\SSS'_t \leftarrow \texttt{Newton-Schulz5}(\SSS_t), \\
    &\M_t = \M_{t-1} + \SSS'_t.
\end{align}
Since the calculation of all $\SSS_t$s can be done in parallel, the calculation of $\texttt{Newton-Schulz5}(\cdot)$ can also be done in parallel.

\head{Architectural Backbone}
As for the architectural backbone, we follow the recent modern recurrent models~\citep{behrouz2024titans, arora2024simple, yang2024parallelizing, Allen2025-canon} and use linear layers to project keys, values, and queries, followed by short convolution layers with size 4. We apply normalization on keys and queries to stabilize the training. We also follow \citet{behrouz2024titans} and use two hybrid variants of MAL and MAG for our \model{} model. The architectures are illustrated in \autoref{fig:arch-vis}. For models with deep memory architectures we use 2-layer MLP with residual connections:
\begin{align}
    \M(\cdot) = (\cdot) + W_1 \sigma(W_2 (\cdot)).
\end{align}
We further extend this memory architecture, which is commonly used in recent studies~\citep{behrouz2024titans, irie2021going, behrouz2025Miras}, to gated MLP layer as:
\begin{align}
    \M(\cdot) = (\cdot) + W_1 \left(\sigma\left(W_2 (\cdot)\right) \otimes W_3 (\cdot) \right),
\end{align}
where $W_1, W_2, W_3$ are linear learnable matrices. We refer to \model{} with the above memory architecture as \model++.

\begin{figure*}
    \centering
    \includegraphics[width=0.9\linewidth]{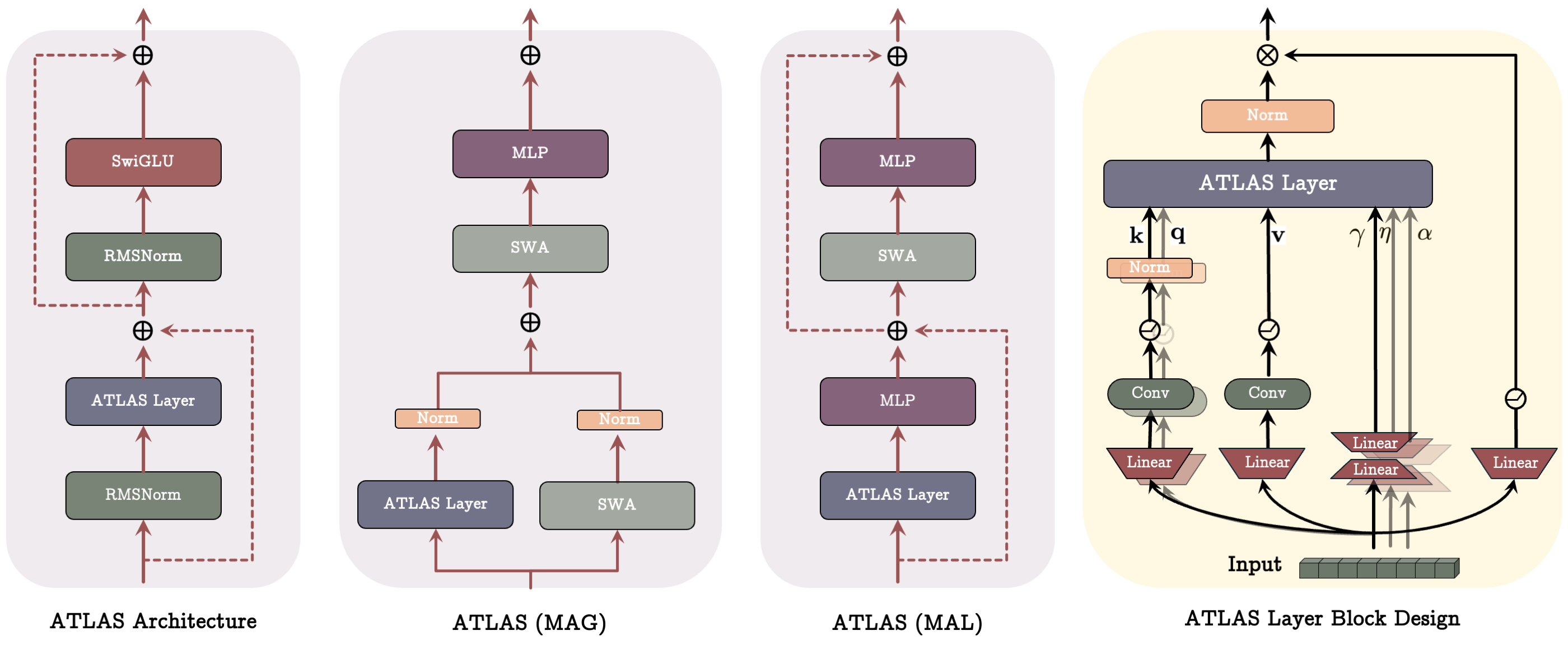}
    \caption{Visualization of the \model's (and our other variants') architecture, and its hybrid counterpart with SWA.}
    \label{fig:arch-vis}
\end{figure*}

\section{Experiments}\label{sec:experiments}
Next, we evaluate the performance of \model, \omodel, \textsc{DeepTransformer}s, and \textsc{Dot} in language modeling, commonsense reasoning, needle in haystack, and in-context recall tasks. Although we also discussed several other variants, such as SWLA, in our experiments we focus on the above models so in addition to comparison with state-of-the-art models, we also answer the following questions: 
\begin{enumerate}
    \item Is deep memory effective for \texttt{softmax} attention? (see \autoref{tab:lm_results} — comparison of Transformer++ and \textsc{DeepTransformer}s)
    \item Does the use of \learningrule{} improve the performance \texttt{softmax} attention? (see \autoref{tab:lm_results} — comparison of Transformer++, \textsc{DeepTransformer}s, and \textsc{Dot})
    \item Does the \learningrule{} rule provide more expressive memory update? (see \autoref{tab:lm_results} and \autoref{tab:ablation} — the performance of \omodel, and \model)
    \item Is locally optimal memory update effective? (see \autoref{tab:lm_results} and \autoref{tab:ablation} — comparison of \omodel, and \model)
    \item Is non-linear feature mapping effective? (see \autoref{tab:ablation})
    \item Can the proposed improvements close the gap with Transformers in in-context recall tasks? (see \autoref{tab:recal})
    \item What is the effect of the internal optimizer on the memory? (see \autoref{fig:learnability-1})
\end{enumerate}

\head{Setup}  We train our models with training context window of size 4K using FineWeb dataset~\citep{penedo2024fineweb}. We use model size of 340M, 400M, 790M, and 1.3B parameters and train them on 15B, 15B, 30B, and 100B tokens sampled from the dataset. Baseline results are reported
by~\citet{yang2024gated, behrouz2024titans, behrouz2025Miras}. Perplexity is measured on held-out validation data. As for the downstream tasks, we evaluate trained models on Wikitext~\citep{merity2017pointer}, LMB~\citep{paperno-etal-2016-lambada}, PIQA~\citep{bisk2020piqa}, HellaSwag~\citep{zellers-etal-2019-hellaswag}, WinoGrande~\citep{sakaguchi2021winogrande},  ARC-easy (ARC-e) and ARC-challenge (ARC-c)~\citep{clark2018think}, SIQA~\citep{sap-etal-2019-social}, and BoolQ~\citep{clark-etal-2019-boolq}. Additional details about the experimental setups and other used datasets are in \autoref{app:exp-details}.

\begin{table*}[t!]
\centering
\caption{
Performance of \model{} and baselines on language modeling and common-sense reasoning tasks. Hybrid models are marked with~$^*$. The best results are highlighted \colorbox{myblue}{highlighted}. 
}\label{tab:lm_results}
\centering
\resizebox{0.9\linewidth}{!}{
\centering
\begin{tabular}{l|c c|c c c c c c c c c}
\toprule
\textbf{Model}  & \textbf{Wiki.}  &  \textbf{LMB.} &  \textbf{LMB.} & \textbf{PIQA} &    \textbf{Hella.} & \textbf{Wino.} & \textbf{ARC-e} &  \textbf{ARC-c} &  \textbf{SIQA}  & \textbf{BoolQ} &  \textbf{Avg.} \\
 & ppl $\downarrow$  &  ppl $\downarrow$  &  acc $\uparrow$  & acc $\uparrow$ &   acc\_n $\uparrow$  & acc $\uparrow$  & acc $\uparrow$ & acc\_n $\uparrow$ &  acc $\uparrow$  & acc $\uparrow$ &   $\uparrow$  \\
\midrule
\midrule
\multicolumn{12}{c}{760M params / 30B tokens} \\
\midrule
 Transformer++ & 25.21 & 27.64 & 35.8 & 66.9 & 42.2 &	51.9 & 60.4	& 32.5 &  39.5  & 60.4  & 48.69 \\
 \textsc{DeepTransformers} (ours) & 20.32 & 20.67 & 36.9 & 68.4 & 49.8 & 52.8 & 65.7 & 34.9 & 40.2 & 61.8 & 51.31 \\
 \textsc{Dot} (ours) & 19.96 & 20.15 & 39.0 & 69.1 & 50.7 & 53.1 & 66.2 & 37.0 & 40.3 & 63.7 & 52.39 \\
 \midrule
 RetNet & 26.08 & 24.45 & 34.5 & 67.2 & 41.6 &	52.1 & 63.2	& 32.8 &  38.4  & 57.9  &  48.46\\
 DeltaNet & 24.37 & 24.60 & 37.1 & 66.9 & 42.0 &	50.7 & 64.9	& 31.4 &  39.9  & 59.0  & 48.97 \\
 TTT & 24.17 & 23.51 & 34.7 & 67.3 & 43.9 & 51.0 & 64.5 & 33.8 & {40.2} & 59.6 & 47.32 \\
 Gated DeltaNet & {21.18} & {22.09} & {35.5} & {68.0} & {44.9} & {50.7} & {66.9}	& {33.1} &  {39.2}  & 59.1  & 49.69 \\
 Samba$^{*}$ & 20.63 & 22.71 & 39.7 & 69.2 & 47.4 &	52.0 & 66.9	& 33.2 &  39.0  & 61.2  & 51.08 \\
  Gated DeltaNet-H2$^*$ &  {19.88} &  20.83 &  {39.2} & 69.0 & {48.2} &	{52.6} & 67.0	& {35.5} &  {39.4}  & 61.1  & 51.49 \\
  Titans (LMM) & 20.04 & 21.96 & 37.4 & 69.3  & 48.5 & 52.3 & 66.3 & 35.8 & 40.1 & 62.8 & 51.56\\
\textsc{Memora} &  22.28 & 22.31 & 38.2 & 67.8  & 49.3 & 53.3 & 63.6 & 36.1 & 40.9 & 63.0 & 51.52\\
  \midrule
  SWDT (ours) &  19.89 & 21.52  &  36.2   &  68.3 &  45.2 &  53.0 & 65.4  & 34.2 & 39.5 & 59.5 &  50.1\\
DLA (ours) &  23.12 & 22.09  &  36.1   &  68.0 &  47.9 &  52.7 & 65.8  & 34.6 & 39.1 & 59.6 &  50.46\\
  \midrule
  \omodel{} (ours)&  19.16 & 20.14  &  38.7   &  69.8 &  50.0 &  53.3 & 67.8  & 36.8 & 39.6 & 64.4 &  52.56\\
\model{} (ours) &   18.92 & 21.01  &  39.1   &  69.7 &  50.2 &  53.5 & 67.5  &  37.1 & 40.7 & 64.3 &  52.77\\
\model++ (ours) &   19.04 & \cellcolor{myblue}20.03  &  39.7   &  69.7 &  \cellcolor{myblue}51.1 &  53.2 & \cellcolor{myblue}68.2  & \cellcolor{myblue} 37.4 & 40.9 & 64.4 &  \cellcolor{myblue}53.09\\
\midrule
\model{} (MAG)&  \cellcolor{myblue}18.62 & 21.18  &  \cellcolor{myblue} 40.0   &  \cellcolor{myblue} 70.3 &  50.5 &  53.0 & 68.1  & 36.5 & \cellcolor{myblue} 41.2 & \cellcolor{myblue} 65.0 &  53.08\\ 
\model{} (MAL) &  19.07 & 21.46  &  38.8   &  69.2 &  50.5 &  \cellcolor{myblue}53.6 & 67.3  & 36.1 & 41.0 & 64.5 &  52.63\\ 
\midrule
\multicolumn{12}{c}{1.3B params / 100B tokens} \\
\midrule
 Transformer++ & 18.53 & 18.32 & 42.6 & 70.0 & 50.2 &	53.5 & 68.8	& 35.1 &  40.7  & 57.1  & 52.25 \\
  \textsc{DeepTransformers} (ours) & 15.67 & 12.63 & 49.4 & 72.6 & 57.0 & 58.8 & 71.1 & 37.5 & 41.6 & 61.5 & 56.19 \\
 \textsc{Dot} (ours) & 15.28 & 11.96 & 50.1 & 73.3 & 57.5 & 60.4 & 72.2 & 41.2 & 42.7 & 61.4 & 57.35 \\
 \midrule
 RetNet & 19.08 & 17.27 & 40.5 & 70.1 & 49.2 &	54.1 & 67.3	& 33.8 &  {40.8}  & {60.4}  & 52.02 \\
 Mamba2 & {16.56} & {12.56} & {45.7} & {71.9} & {55.7} &	{55.2} & {72.5}	& {37.9} &  40.2  & 60.1  & {54.89} \\
 DeltaNet & 17.71 & 16.88 & 42.5 & 70.7 & 50.9 &	53.3 & 68.5	& 35.7 &  40.2  & 55.3  & 52.14 \\
 Gated DeltaNet & {16.42} & {12.17} & {46.6} & {72.2} & {55.8} &{57.4} & {71.2}	& {38.4} &  {40.6}  & 60.2  & {55.32} \\
 Samba$^*$ & 16.13 & 13.29 & 44.9 & 70.9 & 53.4 &	55.6 & 68.8	& 36.2 &  40.0  &  {62.1}  & 54.00 \\
  Gated DeltaNet-H2$^*$ & {15.91} &{12.55} &  {48.8} & {72.2} &  {56.9} &	{57.8} & {71.4}	 &{39.1} &  {41.2}  & 61.6   & {56.18}  \\
Titans (LMM) & 15.60 & 11.41 &  49.1 & 73.1  & 56.3 &  59.8 & 72.4 & 40.8 &  42.1 & 61.0 & 56.82\\
\textsc{Memora} &  15.90 & 12.04  &  48.7   &  73.1 &  56.0 &  57.4 & 71.5  & 37.9 & 40.2 & 61.3 &  55.87\\
\midrule
\omodel{} (ours) & 14.91 & 11.26 & 49.7   & 73.4 & 57.6 & 59.7 & 72.6 & 40.3 & 42.4  & 62.1 & 57.23\\
\model{} (ours)  & 14.97 & 10.98 & 50.1   & \cellcolor{myblue}73.9 & 57.3 & 60.2 & \cellcolor{myblue}72.8 & 41.0 & \cellcolor{myblue}42.9 & \cellcolor{myblue}62.8 & 57.62\\
\model++ (ours)  & \cellcolor{myblue}14.40 & \cellcolor{myblue}10.72 & \cellcolor{myblue}50.8   & 73.5 & \cellcolor{myblue}59.4 & \cellcolor{myblue}61.1 & 71.3 & \cellcolor{myblue}43.7 & 42.5 & 61.9 & \cellcolor{myblue}58.03\\
\bottomrule
\end{tabular}
}
\end{table*}

\subsection{Language Modeling and Common-Sense Reasoning}\label{sec:language-modeling}
The results for \model, and  \omodel{} as well as their corresponding baselines of SWDT, DLA, \textsc{DeepTransformer}s, and \textsc{Dot} with the size of 760M and 1.3B are reported in \autoref{tab:lm_results}.  (see \autoref{app:exp} for the results of small scale). Among non-hybrid models, including Transformer++, our \model, and  \omodel{}  achieve the best performance in both perplexity and accuracy measures. We attribute this performance to their ability to memorize the context rather than individual tokens. Comparing \omodel{} with Titans, that also uses the same momentary objective (i.e., $\ell_2$ loss), but with context window of 1, we can observe the effectiveness of having non-online learning rule. On the other hand, our models, alone without any attention, can outperform hybrid models, while their hybrid variant of MAG further improve their performance. This performance gain is also related to the use of polynomial kernels that enhance the memory capacity of the model. See \autoref{tab:ablation} for a more controlled study on the effect of different components. 

Comparing Transformer++ with our more generalized Transformers (i.e., \textsc{DeepTransformer}s, and \textsc{Dot}) we observe a consistent performance improvement. We attribute this performance to their deep memory, which makes them more powerful to model the dependencies of tokens. Comparing \textsc{Dot} with \textsc{DeepTransformer}s, we can see the advantage of \learningrule{} rule, which helps the model to better manage its memory.

\begin{table}
    \centering
    \caption{Performance of \model{} and baselines on S-NIAH task from RULER benchmark. The best results among {\colorbox{myblue}{simple}} and {\colorbox{mygreen}{hybrid}} models are highlighted.}
    \label{tab:hystack}
    \resizebox{0.8\linewidth}{!}{
    \begin{tabular}{l c c c c c c c c c c c}
    \toprule
    \multirow{2}{*}{Model} & \multicolumn{4}{c}{\textbf{S-NIAH-PK}} & \multicolumn{4}{c}{\textbf{S-NIAH-N}} & \multicolumn{3}{c}{\textbf{S-NIAH-W}} \\
    \cmidrule(lr){2-5} \cmidrule(lr){6-9} \cmidrule(lr){10-12}
    &  2K & 4K & 8K & 16K &  2K & 4K & 8K & 16K &  2K & 4K & 8K \\
    \midrule
    \midrule
       TTT  & 98.4 & 98.8 & 98.0 & 88.4 & 60.2 & 36.6 &  10.2 & 4.4 & 78.8 & 28.0 & 4.4 \\
       DeltaNet & 96.8 & 98.8 & \cellcolor{myblue}98.6 & 71.4 & 47.2 & 15.4 & 12.8 & 5.4 & 46.2 & 20.0 & 1.6 \\
       Titans (LMM) & 99.8 & 98.4 & 98.2 & 96.2 & \cellcolor{myblue}100.0 & 99.8 &  \cellcolor{myblue}93.4 & 80.2 & 90.4 & 89.4 & 85.8 \\
       \model{} & \cellcolor{myblue}100 & \cellcolor{myblue}99.2 & 98.0 & \cellcolor{myblue}97.0 & \cellcolor{myblue}100.0 & \cellcolor{myblue}100.0 & 93.0 & \cellcolor{myblue}84.0 & \cellcolor{myblue}93.2 & \cellcolor{myblue}90.6 & \cellcolor{myblue}86.2 \\
       \midrule
        Samba & 98.8 & 98.0 & 97.4 & 97.2 & 98.8 & 98.6 & 96.2 & 95.6 & 96.8 & 90.0 & 84.0\\
       Gated DeltaNet-H2$^*$ & 99.2 & 97.8 & 97.4 & 98.4 & 98.0 & 97.8 & 96.2 & 95.8 & 97.4 & 96.8 & 88.4 \\
       \model{} (MAG) & \cellcolor{mygreen}100 & \cellcolor{mygreen}100 & 99.4 & \cellcolor{mygreen}98.6 & \cellcolor{mygreen}100 & \cellcolor{mygreen}99.2 &  97.4 & \cellcolor{mygreen}97.0 &  \cellcolor{mygreen}99.4 & 98.2  & 92.4  \\
       \model{} (MAL) & 99.8 & 99.6 & 98.4 & 96.8 & 99.8  & 98.0 & 97.2  & 96.8 & 98.0 & \cellcolor{mygreen}98.4 & 92.6\\
       \textsc{DeepTransformers} & \cellcolor{mygreen}100 & \cellcolor{mygreen}100 & 98.2 & 97.8 & \cellcolor{mygreen} 100 & 98.8 & \cellcolor{mygreen} 97.8 & 94.0 & 95.8 & 92.2 & 88.4 \\
       \textsc{Dot} & \cellcolor{mygreen}100 & \cellcolor{mygreen}100 & \cellcolor{mygreen}99.6 & \cellcolor{mygreen}98.6 & \cellcolor{mygreen} 100 & 100 & \cellcolor{mygreen} 97.8 & 96.8 & 99.0 & \cellcolor{mygreen}98.4 & \cellcolor{mygreen}93.2 \\
    \toprule
    \end{tabular}
    }
\end{table}

\begin{minipage}[t!]{\textwidth}
  \begin{minipage}[t]{0.49\textwidth}
    \centering
    \includegraphics[width=0.8\linewidth]{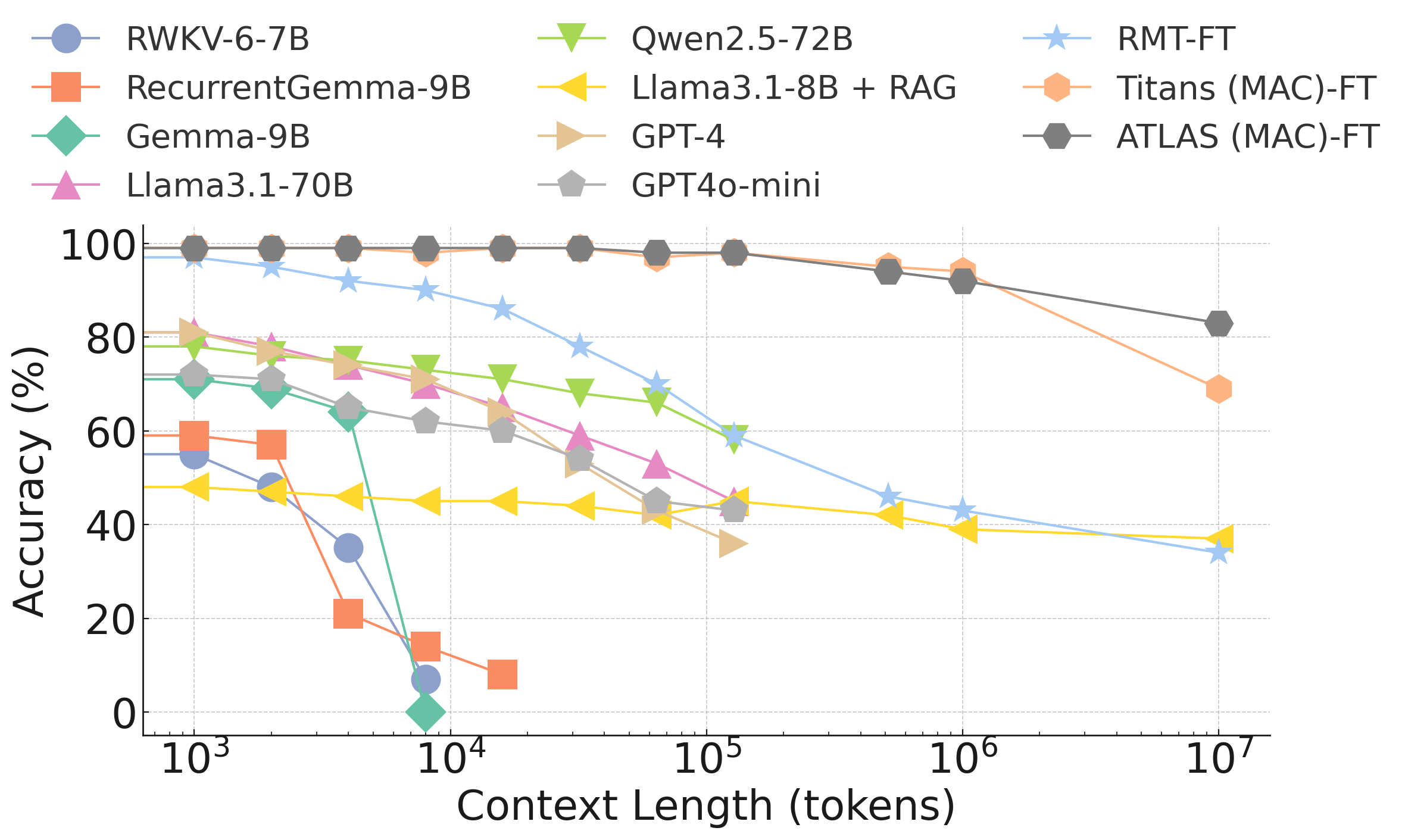}
    \captionof{figure}{Performance of \model{} and baselines on BABILong benchmark. \model{} surpasses Titans performance and effectively scale to 10M context length in this task.}
    \label{fig:babilong}
    \end{minipage}~
    \hfill~
    \begin{minipage}[t]{0.49\textwidth}
    \includegraphics[width=0.8\linewidth]{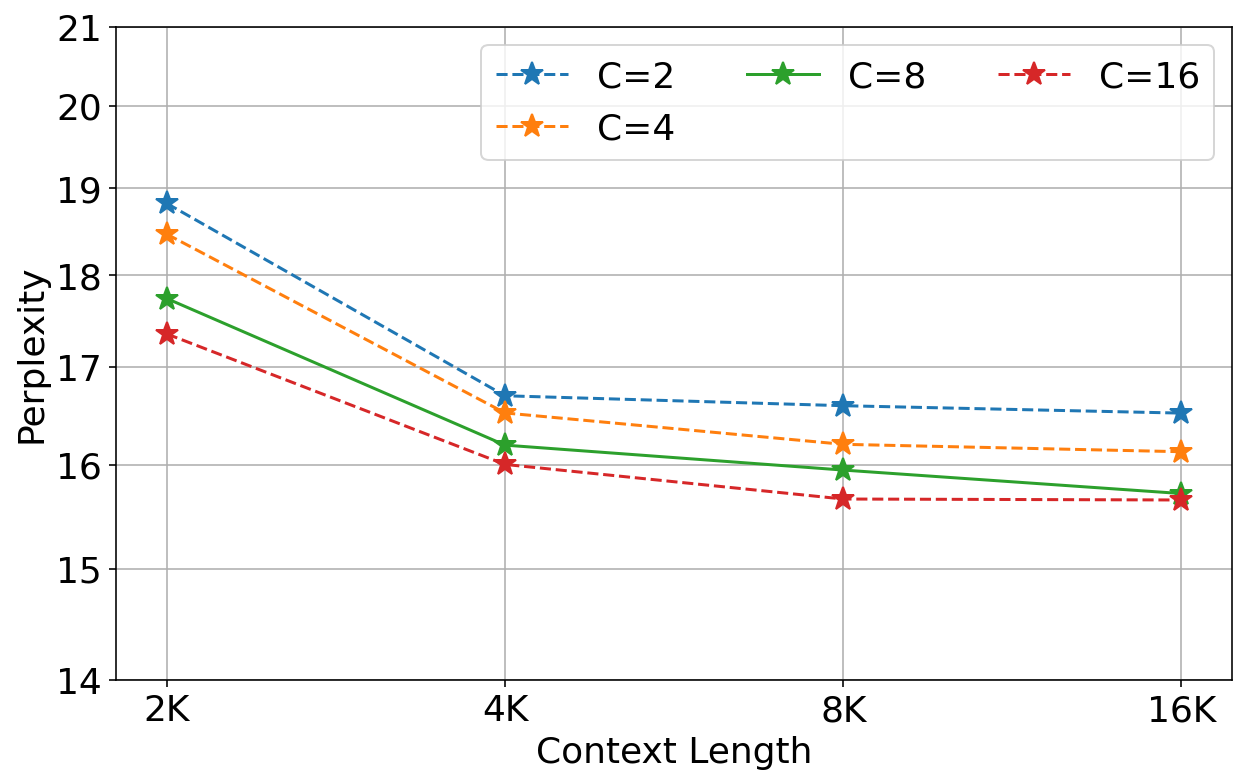}
    \captionof{figure}{The effect of local context length (i.e. $c$) on the performance of \omodel{} with different global context length.}
    \label{fig:effect-c}
    \end{minipage}
\end{minipage}

\subsection{Long Context: Needle In a Haystack}\label{sec:niah}
One of our main motivations to design \model{} is to enhance the performance of long-term neural memory module in long context tasks. Accordingly, to evaluate the effectiveness of our designs for improving the effective context length and memory capacity, we perform an experiment on needle-in-haystack tasks of RULER~\citep{hsieh2024ruler} benchmark. The performance of \model{} and its hybrid variants, as well as our Transformer-like architectures and baselines are reported in \autoref{tab:hystack}. \model{} shows very good performance compared to the recurrent baselines, outperforming modern recurrent neural networks such as Titans and DeltaNet. Its hybrid variants further improve its effective context length, effectively extrapolating to sequences with $\times 4$ of their training context size. We attribute this performance to the proposed enhancements for the capacity of the memory. We further perform ablation studies to validate this claim. Also, our Transformer-like architectures outperforms the baselines, even our hybrid variants of \model{} in longer contexts. This shows the importance of exponential feature mapping in longer sequences.

\subsection{Long Context: BABILong Benchmark}\label{sec:babilong}
To compare the effectiveness of \model{} with Titans \citep{behrouz2024titans} in ultra-large sequences, we further evaluate \model's performance on BABILong benchmark~\citep{kuratov2024babilong}. In this experiment, we follow \citet{behrouz2024titans} and use MAC architecture but without persistent memory tokens. We also follow the original setup in the benchmark and fine-tune our model. The results are reported in \autoref{fig:babilong}. While \model{} shows competitive and on par performance with Titans until 1M context length, the performance of Titans drops in 10M. \model, however, maintains its performance and achieve +80\% accuracy in 10M context length. We attribute this to more powerful memory; in terms of (1) memory management (i.e., the use of Muon), (2) better memory capacity due to polynomial kernels, and (3) its nature to memorize the context, instead of individual tokens.

In previous sections, we show the effectiveness of our Transformer-like architectures (i.e., \textsc{DeepTransformer}s and \textsc{Dot}) in both language modeling and long-context needle-in-haystack tasks. From now on, we focus on our recurrent architectures (i.e., \model, and \omodel) to show the importance of presented improvements.

\subsection{Learnability Experiments}
We have also performed some small-scale experiments to analyze the function-learning capability of small MLPs in an online fashion. In this setting, we have a sequence of tuples $(i_1, o_1), \ldots (i_t, o_t)$ with both $i_j, o_j \in \mathbb{R}^d$ for all $j$. We train an MLP $\mathcal{M}$ in an online fashion to minimize $\text{loss}_j = \|{i_j - o_j}\|_2^2 / \|{o_j}\|_2^2$ -- specifically, we compute the gradient at time step $j$ as $\nabla_{\mathcal{M}\text{.params}} \text{loss}_j$  and use standard optimizers such as Adam, Rmsprop and SGD to update the parameters. Such experiments help us understand the representation power of the models we use to represent memory and the power of optimization algorithms to quickly learn the underlying sequence mapping.

\begin{figure*}
  \begin{center}
  \begin{subfigure}[t]{.39\textwidth}
    \centering
    \includegraphics[width=\linewidth]{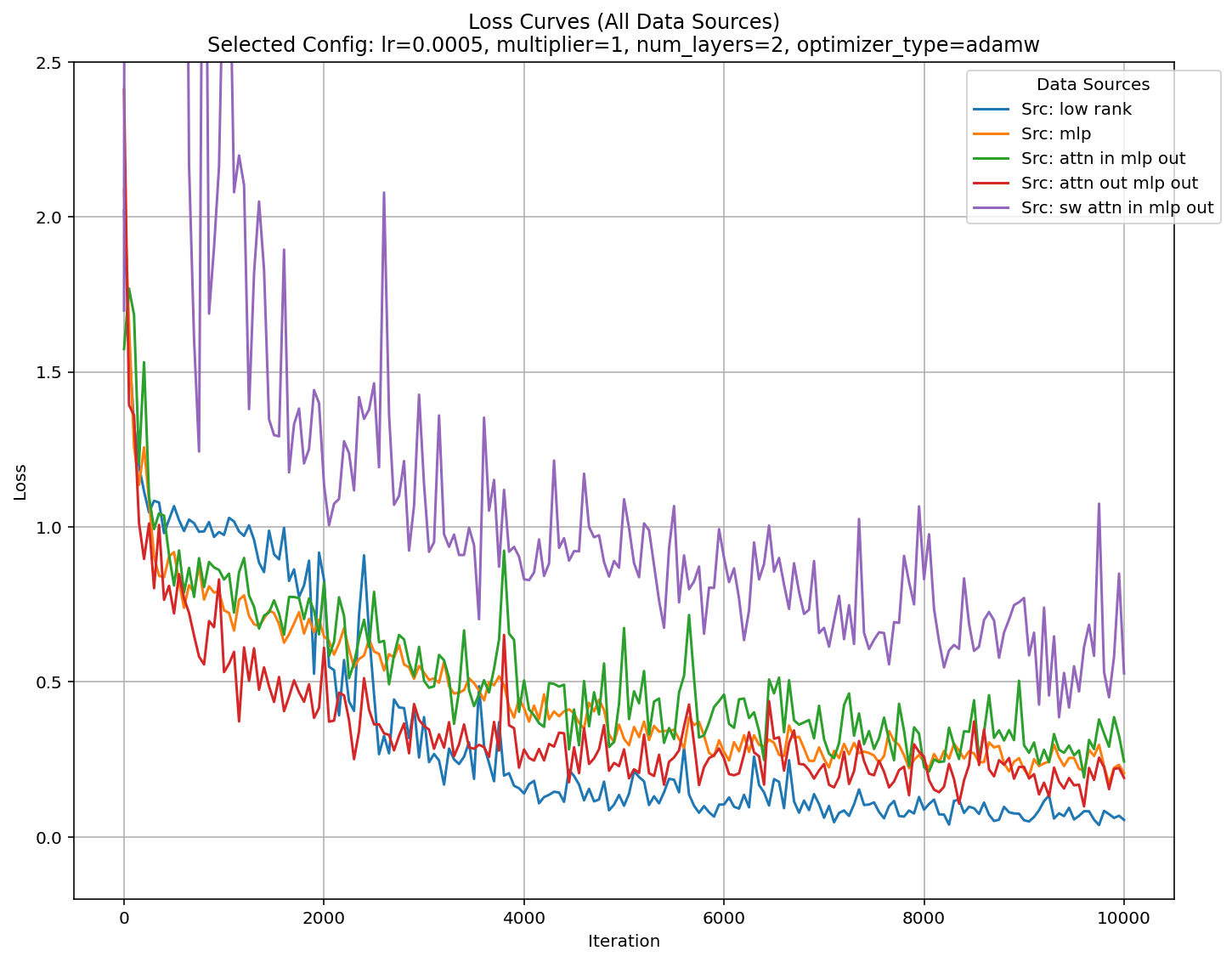}
    \caption{$\mathcal{M}$ with 2 hidden layers and no expansion.}
  \end{subfigure}
    \hspace{3ex}
  \begin{subfigure}[t]{.39\textwidth}
    \centering
    \includegraphics[width=\linewidth]{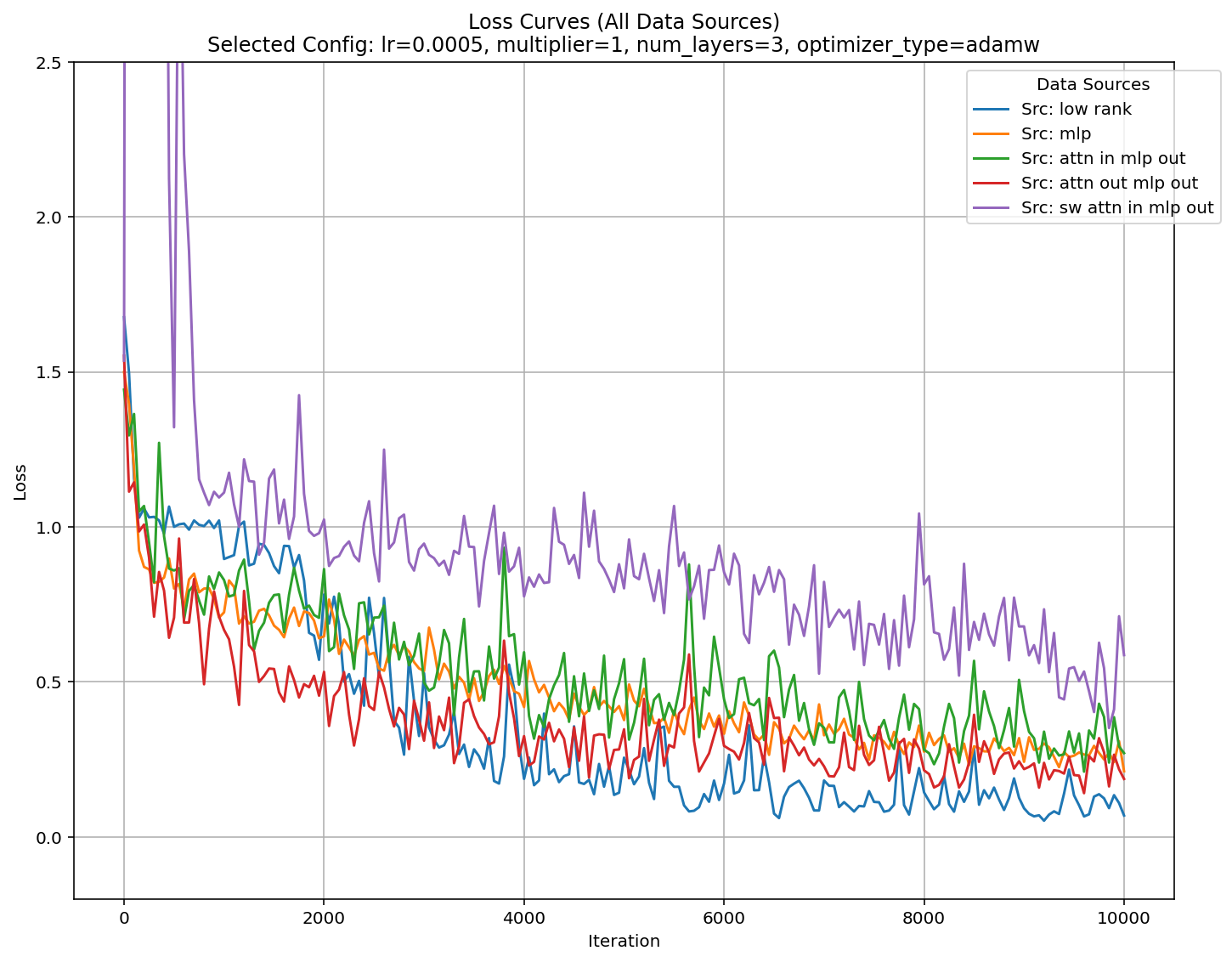}
    \caption{$\mathcal{M}$ with 3 hidden layers and no expansion.}
  \end{subfigure}
  \medskip
  \begin{subfigure}[t]{.39\textwidth}
    \centering
    \includegraphics[width=\linewidth]{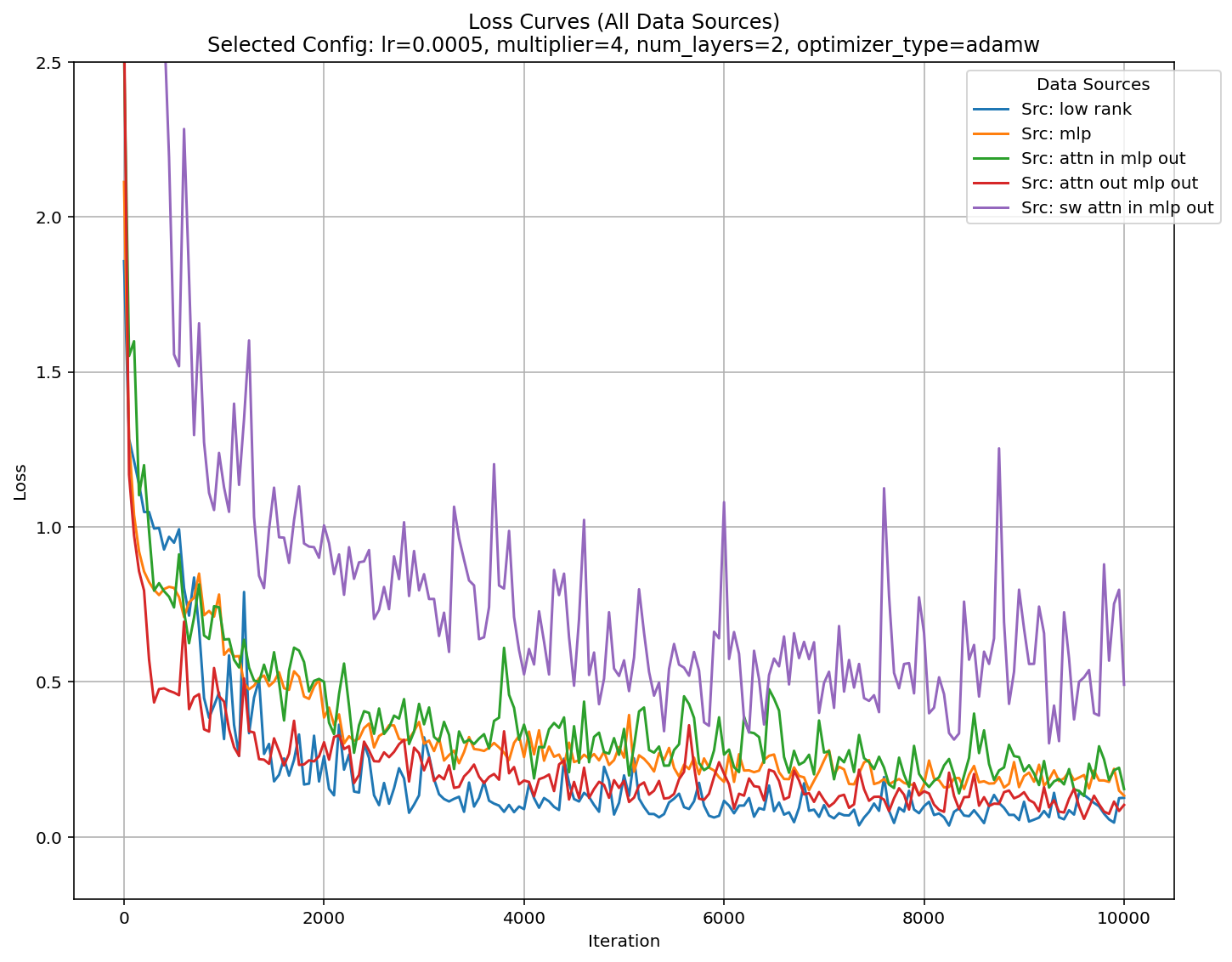}
    \caption{$\mathcal{M}$ with 2 hidden layers and 4x expansion.}
  \end{subfigure}
  \hspace{3ex}
  \begin{subfigure}[t]{.39\textwidth}
    \centering
    \includegraphics[width=\linewidth]{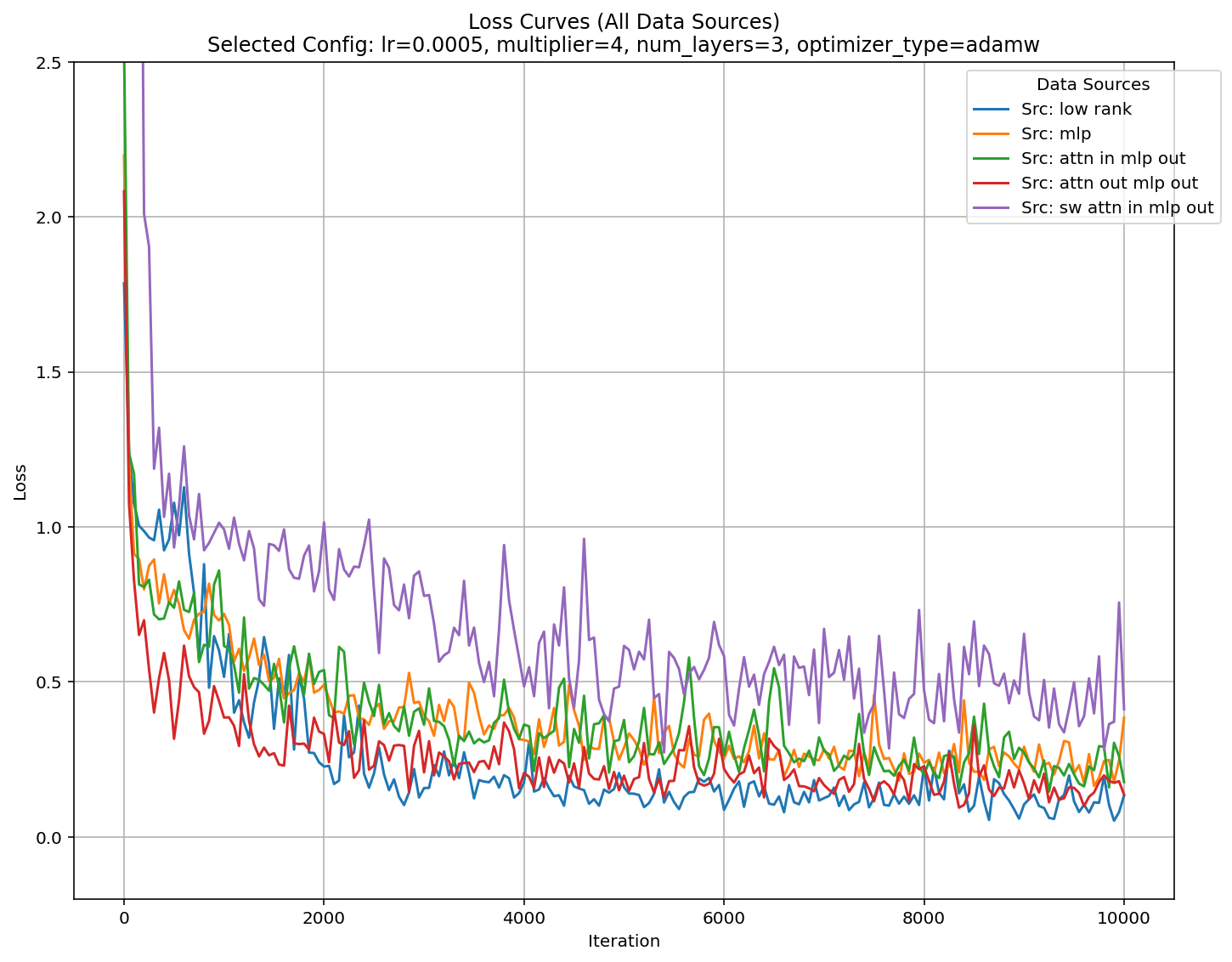}
    \caption{$\mathcal{M}$ with 3 hidden layers and 4x expansion.}
  \end{subfigure}
  \vspace{-2ex}
  \caption{Loss curves for different setting with various hyperparameters}
  \label{fig:learnability-1}
  \end{center}
\end{figure*}

We study five different sequence to sequence functions:
\begin{enumerate}
    \item \textbf{Low Rank Mappings}: We sample a random low rank matrix $\mathbf{W} = \mathbf{X}\mathbf{Y}$ with $\mathbf{X} \in \mathbb{R}^{d \times k}$ and $\mathbf{Y} \in \mathbb{R}^{k \times d}$. We then sample $i_1, \ldots, i_t$ randomly from a Gaussian distribution and set $o_j = \mathbf{W}^{\mathsf{T}} \cdot i_j$ for all $j \in [t]$.
    \item \textbf{MLP Mappings}: We sample an MLP $\mathcal{M}$ with 1 input, 1 hidden and 1 output layer which uses GELU non-linearity. We set the hidden dimension to $d$ so that there is no expansion. We then sample $i_1, \ldots, i_t$ randomly from a Gaussian distribution and then set $o_j = \mathcal{M}(i_j)$ for all $j \in [t]$.
    \item \textbf{Attention+MLP Mapping}: We sample $(i_1, \ldots, i_t)$ from a Gaussian distribution and an MLP $\mathcal{M}$ as above. We additionally sample three $d \times d$ matrices $\mathbf{W_Q}$, $\mathbf{W_K}$ and $\mathbf{W_V}$ and compute $q_j = \mathbf{W_Q}^{\mathsf{T}} \cdot i_j$, $k_j = \mathbf{W_K}^{\mathsf{T}} \cdot i_j$ and $v_j = \mathbf{W_K}^{\mathsf{T}} \cdot i_j$ for all $j \in [t]$. We then compute $o'_1, \ldots, o'_t$ as outputs of the causal masked attention mechanism applied on $\{q_j\}_{j \in [t]}, \{k_j\}_{j \in [t]}, \{v_j\}_{j \in [t]}$ and finally compute $o_j = \mathcal{M}(o_j)$.
    \item \textbf{Attention Outputs as Inputs}: We do the same as above except that we output $o'_j$ as the input sequence and $o_j$ as the output sequence.
    \item \textbf{Sliding Window Attention + MLP Mapping}: We do the same as in \textbf{Attention + MLP Mapping} setting except that we use a sliding window attention instead of full attention. We use a sliding window of 512 in our experiments.
\end{enumerate}
Note that the settings 3 and 5 are much harder to learn since they require (partially) memorizing the previous inputs and outputs to be able to learn the function that maps $i_j$ to $o_j$, whereas the settings 1, 2 and 4 do not need to memorize the previous input-output pairs and just need to learn the underlying low-rank matrix or the MLP that maps the inputs to outputs.

The setting 4 is slightly different to setting 2 in that the inputs are not-independent at each time step and are correlated by the attention mechanism we use to compute the inputs. Thus a strong learning algorithm maybe able to utilize the underlying correlations to learn the mapping faster in setting 4 versus setting 2.

\begin{table*}
    \centering
    \caption{Performance of \model, \omodel, and baselines on the synthetic benchmark of MAD~\citep{poli2024mechanistic}. \model{} outperforms all the baselines, including Transformers.}
    \label{tab:MAD}
    \resizebox{0.8\linewidth}{!}{
    \begin{tabular}{l c c c c c c}
    \toprule
         &  \multirow{2}{*}{Compression} & \multirow{2}{*}{(Noisy) ICR} & \multirow{2}{*}{Fuzzy ICR} & Selective & \multirow{2}{*}{Memorization}  & \multirow{2}{*}{Average}\\
         &  & & & Copying & \\
         \midrule
         \midrule
        Transformers & 49.4 & 100 & 48.2 & 95.9 & 83.8 & 75.46\\
         Gated DeltaNet & 44.8 & 100 & 32.5 & 96.2 & 81.7 & 71.04\\
         Titans & 49.6 & 100 & 49.7 & 99.4 & 83.5 & 76.44\\
         \midrule
         \omodel{} (ours) & 50.9 & 100 & 54.2 & \cellcolor{myblue} 99.6 & 90.2 & 78.98\\
         \model{} (ours) & \cellcolor{myblue}51.6 & \cellcolor{myblue} 100 &  \cellcolor{myblue} 54.9 & \cellcolor{myblue} 99.6 & \cellcolor{myblue} 91.4 & \cellcolor{myblue}79.50\\
    \toprule
    \end{tabular}
    }
\end{table*}

\begin{minipage}[t!]{\textwidth}
  \begin{minipage}[t]{0.60\textwidth}
    \centering
    \captionof{table}{The performance of our models (\model, and \omodel) compared to baselines. While still Transformers achieve the best results in in-context recall tasks, our design of context memorization and polynomial feature maps can close the gap with Transformers. }
    \label{tab:recal}
    \resizebox{\linewidth}{!}{
    \begin{tabular}{l c c c c c c c}
    \toprule
         &  \multirow{1}{*}{SWDE} & \multirow{1}{*}{NQ} & \multirow{1}{*}{DROP} & FDA & SQUAD & \multirow{1}{*}{TQA} & Average \\
         \midrule
         \midrule
        Transformers & \cellcolor{myblue}84.9 & \cellcolor{myblue}23.0 & \cellcolor{myblue}28.4 & \cellcolor{myblue}72.5 & \cellcolor{myblue}48.1 & \cellcolor{myblue}64.4 & \cellcolor{myblue} 53.55\\
         Gated DeltaNet & 63.2 & 19.1 & 26.7 & 33.4 & 39.6 & 59.7 & 40.28\\
         Titans & 65.1 & 20.7 & 27.2 & 37.3 &  42.6 & 61.0 & 42.31\\
         \midrule
         \omodel{} (ours) & \underline{67.4} & 21.1  &  27.2  & 39.0 & 43.2 & 60.9 & 43.13\\
         \model{} (ours) & 66.8  & \underline{21.9}  &  \underline{27.4}  & \underline{40.7} & \underline{44.1} & \underline{61.3} & \underline{43.70}\\
    \toprule
    \end{tabular}
    }
  \end{minipage}~\
  \hfill~
  \begin{minipage}[t]{0.40\textwidth}
    \centering
    \captionof{table}{Ablation Study on \model. All components of \model{} are positively contributing to its performance. }
    \label{tab:ablation}
    \resizebox{\linewidth}{!}{
    \begin{tabular}{l c c}
    \toprule
    \multirow{2}{*}{Model}     & Language Modeling & C.S. Reasoning \\
    & ppl $\downarrow$ & acc $\uparrow$ \\
    \midrule
    \midrule
    \model{}         &   19.97  &  52.77  \\ 
    \midrule
    \hspace{6pt}+{Gated MLP Memory} & 19.53 & 53.09                \\
    \midrule
    \hspace{6pt}+\texttt{Attn} (MAG) & 19.90 & 53.08                \\
    \hspace{6pt}+\texttt{Attn} (MAL) & 20.26 & 52.63               \\
    \midrule
    Linear Memory          & 21.03   &  49.74    \\
    w/o Muon      & 19.65   &  52.56    \\
    $c = 1$                & 21.98   &  49.26    \\
    w/o Polynomial Mapping  &  22.14   &  50.57   \\
    \toprule
    \end{tabular}
    }
  \end{minipage}
\end{minipage}

We set $d = 256$ and show the loss curves vs sequence position for all the five settings with function learning MLP $\mathcal{M}$ being defined and trained with different settings in Figure~\ref{fig:learnability-1}. We can see that in all the settings, the model learns non-trivial mappings from inputs to outputs with the $loss_j = \|i_j - o_j\|_2^2/\|o_j\|_2^2$ being smaller than $1$ eventually. Most notably, the correlations in inputs in setting 4 induced by the attention mechanism makes the model quickly learn the mapping compared to in setting 2 and the models usually learn the best in setting 1 which is the least complex function.  

The models do the worst in settings 3 and 5 which require the models to (partially) memorize the inputs and outputs to learn the attention mechanism outputs. Surprisingly, the models learn to do better in setting 3 vs setting 5, when we would expect that capacity requirement for setting 3 to be higher than setting 5. We hypothesize that the learning algorithm is unable to make the model `forget' old inputs which makes the loss worse in sliding window setting when compared to global attention setting. A caveat of our analysis is that, the attention computation is done on randomly initialized vectors and hence the attention matrix is usually not spiky, unlike in the attention matrix for trained set of query, key and value vectors in LLMs. This leads to attention outputs being close to the mean of value vectors in the context.

\subsection{Additional Experiments: In-context Recall, MAD Synthetic Benchmark, and Associative Recall}\label{sec:retrieval-tasks}
In this section, we first evaluate the performance of our models on MAD benchmark, a synthetic benchmark that evaluate the performance of models in recall, memorization, compression, and copying tasks~\citep{poli2024mechanistic}. The results are reported in \autoref{tab:MAD}. \model{} achieves the best results in all aspects, particularly in memorization, which shows the importance of its components for enhancing the memory capacity.

In-context recall tasks is one of the most challenging benchmarks for recurrent neural networks. In this section, we follow \citet{arora2024simple} and perform experiments on SWDE~\citep{lockard2019openceres}, NQ~\citep{kwiatkowski2019natural}, DROP~\citep{dua2019drop}, FDA~\citep{arora2023language}, SQUAD~\citep{rajpurkar2016squad}, and TQA~\citep{kembhavi2017you} to evaluate and compare the performance of \model{} with baselines and Transformers. The results are reported in \autoref{tab:recal}. While Transformers still achieve the best results in in-context recall tasks, \model{} and \omodel{} shows competitive performance and performs better than state-of-the-art recurrent models. We again attribute this performance to better memory management and capacity.

\begin{minipage}[t!]{\textwidth}
    \begin{minipage}[t]{0.33\textwidth}
    \includegraphics[width=\linewidth]{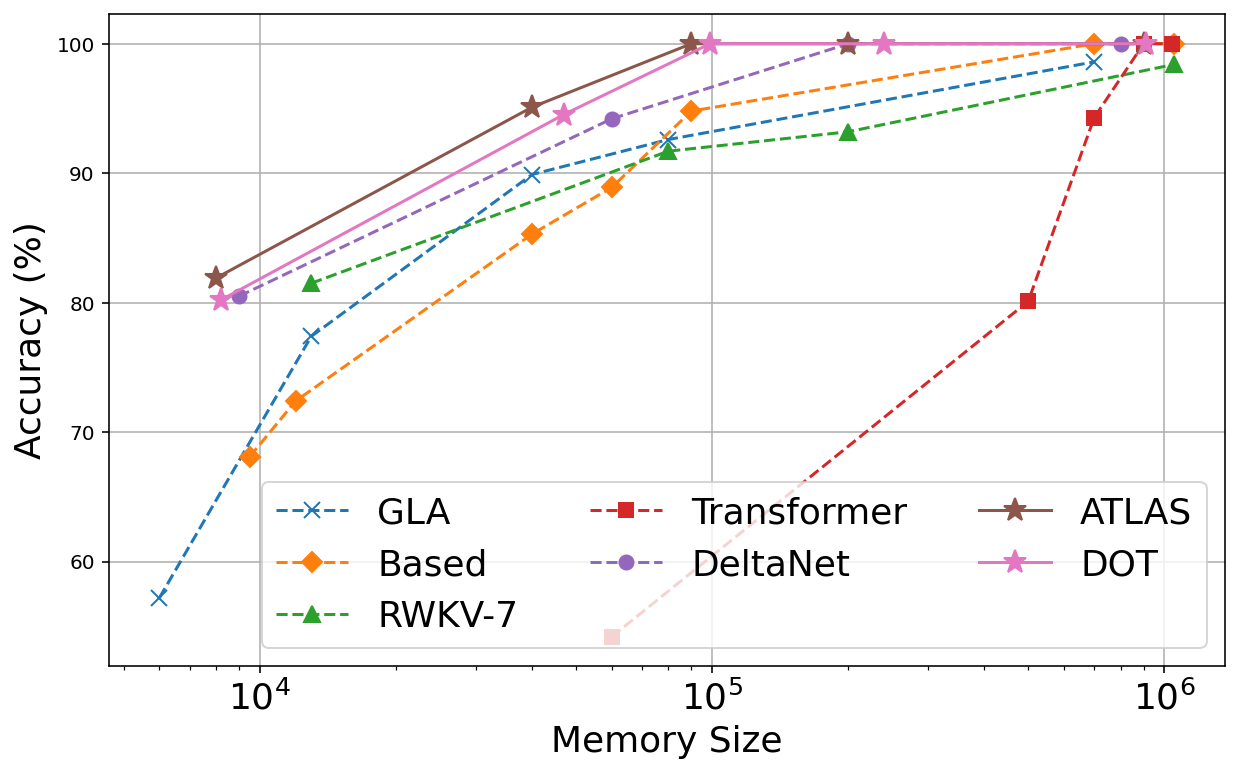}
    \captionof{figure}{The results for associative memory recall.}
    \label{fig:AMR}
    \end{minipage}~
    \hfill~
      \begin{minipage}[t]{0.66\textwidth}
    \centering
    \includegraphics[width=0.5\linewidth]{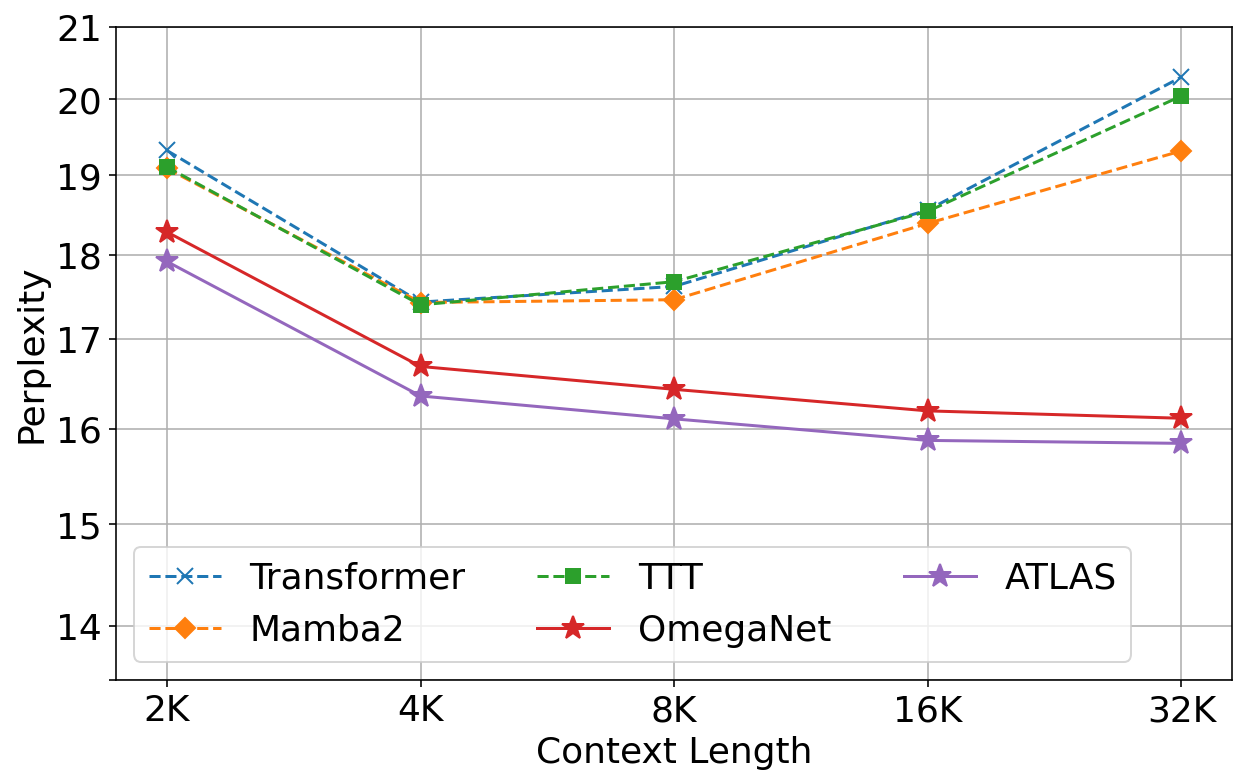}~
    \includegraphics[width=0.5\linewidth]{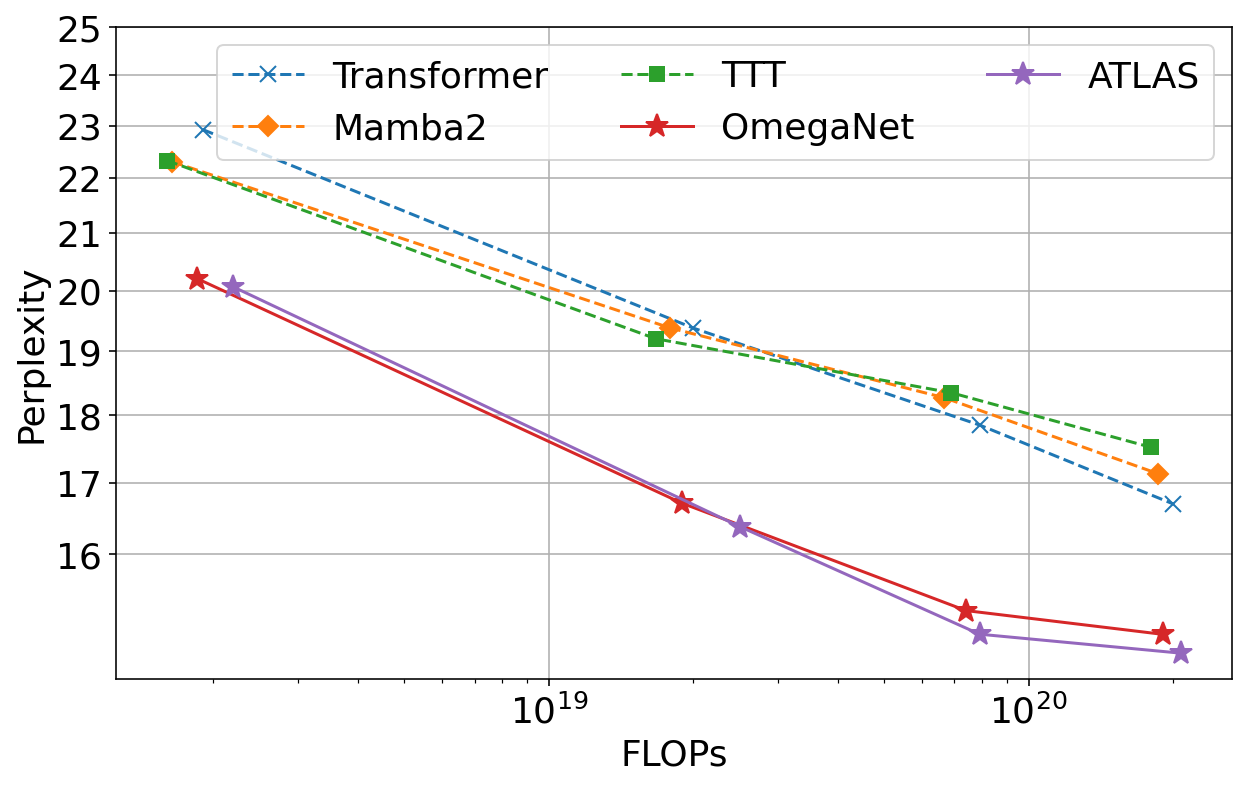}
    \captionof{figure}{Scaling patterns of \model, and \omodel{} with respect to (Left) training context length, and (Right) FLOPs. }
    \label{fig:scaling}
    \end{minipage}
\end{minipage}

Finally, following \citet{yang2024parallelizing} and \citet{arora2023zoology} we evaluate the performance of \model{} and \textsc{Dot} in Multi-Query Associative Recall (MQAR) task~\citep{arora2023zoology}. The results are reported in \autoref{fig:AMR}. Both models show good performance compared to baselines and \model{} achieve the best performance per memory size compared to state-of-the-art models such as DeltaNet~\citep{yang2024parallelizing}.

\subsection{Ablation Study and Scaling Patterns}
In this section, we perform an ablation study on the differernt components of \model, and also evaluate its scaling patterns with respect to the number of parameters and also the context length of the training. The results for ablation study are reported in \autoref{tab:ablation}. The results show that: (1) more powerful memory architectures such as gated MLP can further enhance the performance of \model; (2) The hybrid variants further improve the performance, where MAG shows better improvement compared to MAL architecture; (3) Polynomial mappings as well as deep memory are particularly important when we use context memorization (i.e., \learningrule{} rule). \autoref{fig:effect-c} also shows the effect of local context length (i.e., $c$) on the performance of the model. With the increase of $c$ we can achieve better performance, mainly due to the gating parameters of $\gamma$ that can prune the context, whenever it is needed.

\head{Model Size}
\autoref{fig:scaling} shows the scaling pattern of \model, and \omodel, with respect to number of parameters and compared to baseline. Both models achieve a good scaling pattern with increasing the model size, achieving lower perplexity in all scales compared to baselines.

\head{Context Length}
\autoref{fig:scaling} shows the scaling pattern of \model, and \omodel, with respect to the context length and compared to baseline. Both models due to high memory capacity can scale well, when increasing the context length.

\section{Conclusion}\label{sec:concolusion}

We introduced \model{}, a new long-term memory module designed to address the core limitations of modern recurrent models in long-context understanding: limited memory capacity, online-only updates, and weak memory management. Our proposed sliding window learning rule, higher-order feature mappings, and advanced memory optimizers offer a principled and scalable approach to overcoming these challenges. Empirically, our models—\omodel, \model, \textsc{DeepTransformer}s, and \textsc{Dot}—achieve consistent improvements over Transformers and recent RNN variants across diverse benchmarks. Theoretically, we provided insight into memory capacity and optimization dynamics, offering explanations for the context length limitations observed in prior works.

\newpage
\printbibliography
\newpage
\appendix

\section{Additional Related Work}\label{app:rw}

\head{Modern Linear Recurrent Neural Networks\footnote{Note that here the term ``linear'' refers to their fast training and inference procedures. This does not refer to their recurrence formula as some models like Titans~\citep{behrouz2024titans}, \textsc{Yaad, Moneta, Memora}~\citep{behrouz2025Miras}, and TTT~\citep{sun2024learning} are based on \emph{non-linear} recurrence but fast at training and inference.}}
Recent research endeavors have concentrated on mitigating the quadratic computational complexity and inherent limitations of Transformer models in processing long-context sequences. This has led to the development of efficient recurrent alternatives, primarily motivated by their rapid inference and training capabilities~\citep{tiezzi2024resurgence}. Initial advancements in this domain, exemplified by models such as RetNet~\citep{sun2023retentive}, RWKV~\citep{peng2023rwkv}, and S5~\citep{smith2023simplified}, employed data-independent transition matrices coupled with Hebbian-like update mechanisms. Subsequently, a second generation of models emerged, incorporating input-dependent parameters within these linear architectures (e.g., linear RNNs~\citep{hasani2023liquid, smith2023simplified}, RWKV6~\citep{peng2024eagle}). These models also explored more expressive memory updating rules, notably those based on the delta rule~\citep{peng2025rwkv7, schlag2021linear, yang2024parallelizing, yang2024gated, liu2024longhorn}. Further evolution in this line of research has extended these memory architectures to deeper models, while concurrently utilizing delta-rule-like update mechanisms~\citep{sun2024learning} or data-dependent momentum-based update rules with forget gating~\citep{behrouz2024titans}. More recently, to augment the performance of delta-rule-based sequential models, \citet{siems2025deltaproduct} have proposed the application of multiple gradient descent updates per token, thereby yielding more expressive sequence models, particularly in state tracking tasks. In addition to the above fast linear recurrent sequence models, several studies have focused on RNNs with non-linear recurrence~\citep{behrouz2025Miras, csordas2024recurrent, merrill2024the, lim2024parallelizing,  schone2025implicit, karami2025lattice, von2023uncovering, gonzalez2024towards}, and how their training can be faster~\citep{gonzalez2024towards, lim2024parallelizing, schone2025implicit}.

\head{Fast Weight Programs}
The conceptualization of linear layers as key-value associative memory systems can be traced back to Hopfield networks~\citep{hopfield1982neural}. This concept was subsequently developed in the context of fast weight programmers, wherein dynamic fast programs are integrated into recurrent neural networks to serve as writable memory stores~\citep{schlag2021linear, schmidhuber1992learning, schmidhuber1993reducing}. Among the learning paradigms for such systems, Hebbian learning~\citep{hebb2005organization} and the delta rule~\citep{prados1989neural} have emerged as the most prominent. Both learning rules have been the subject of extensive investigation within the existing literature~\citep{munkhdalai2017neural, schmidhuber1992learning, munkhdalai2019metalearned, schlag2021linear, irie2021going, yang2024parallelizing, yang2024gated}.

\head{Hopfield Networks} 
Our formulation is architecturally founded upon the broad concept of associative memory, wherein the primary objective is to learn an underlying mapping between keys and values. Seminal work by \citet{hopfield1982neural} on Hopfield Networks introduced one of the earliest neural architectures explicitly based on associative memory, defining it through the minimization of an energy function for storing key-value pairs. Although traditional Hopfield networks have seen diminished applicability in recent years, primarily due to constraints in vector-valued memory capacity and the nature of their energy function, several contemporary studies have focused on enhancing their capacity through various methodologies. These include efforts by \citet{krotov2021hierarchical}, \citet{li2024expressive}, and \citet{krotov2016dense}. Notably, extensions to the energy function of these models, often incorporating exponential kernels, have been explored~\citep{krotov2016dense, lucibello2024exponential}. Furthermore, the relationship between these modernized Hopfield networks and Transformer architectures has been a subject of recent investigation~\citep{ramsauer2021hopfield, hu2024provably}.

\section{\textsc{Miras} Framework}\label{app:miras}
As discussed earlier, \citet{behrouz2025Miras} formalized the concept of associative memory as:
\begin{dfn}[\citet{behrouz2025Miras}] \label{dfn:associative-memory2}
Given a set of keys $\mathcal{K} \subseteq \mathbb{R}^{d_k}$ and values $\mathcal{V} \subseteq \mathbb{R}^{d_v}$, associative memory is an mapping $\M: \mathcal{K} \rightarrow \mathcal{V}$. Learning the associative memory is based on an objective $\mathcal{L}$, called \emph{Attentional Bias}, that determines the type of memory and its priorities:
\begin{align}\label{eq:attentional-bias-loss2}
    \M^* = \arg\min_{\M}\quad \mathcal{L}(\M(\mathcal{K}); \mathcal{V}).
\end{align}
\end{dfn}
Optimizing this objective using an iterative algorithm (e.g., gradient descent) results in the memory update rule. Thus, the sequence model is a meta in-context learner with two  optimization levels:
\begin{enumerate}
    \item \textcolor{c3}{Inner Loop}: Where parameters of the memory module are optimized (i.e., $\boldsymbol{\theta}_{\M} = \{W_1, W_2, \dots, W_{\mathcal{L}_{\M}, \dots}\}$). In the inner optimization loop, all other parameters from the model are considered hyperparameters and are fixed and \emph{not} optimized.   
    \item \textcolor{c3}{Outer Loop}: Where all other parameters of the model are optimized, such as linear projections, MLP layers, convolutions, etc. 
\end{enumerate}

\subsection{Examples}
As an example, one can define the linear attention as the optimization of dot-product similarity with gradient descent: i.e., $\tilde{\ell}_t := \inner{\M_{t-1} \vk_t}{\vv_t}$. That is,
\begin{align}
    \M_t &= \M_{t-1} - \eta_t \nabla  \tilde{\ell}_t(\M_{t-1}; \vk_t, \vv_t) = \M_{t-1} - \eta_t \nabla  \inner{\M_{t-1} \vk_t}{\vv_t} \\
    &= \M_{t-1} + \eta_t \vv_t \vk_t^{\top}.
\end{align}
As an another example, if we use regression loss, instead of the dot-product similarity, we can obtain the DeltaNet~\citep{schlag2021linear}:
\begin{align}
    \M_t &= \M_{t-1}  - \eta_t \nabla \| \M_t \vk_t - \vv_t \|^{2}_2 = \mathbf{I} - \eta_t \vk_t \vk_t^{\top} \M_{t-1} + \vv_t \vk_t^{\top}.
\end{align}

\section{Supporting Proofs}
\Capacity*
\begin{proof}
Let $K = [\vk_1 \cdots \vk_m] \in \mathbb{R}^{d_k \times m}$ and $V = [\vv_1 \cdots \vv_m] \in \mathbb{R}^{d_v \times m}$. The optimization problem becomes minimizing the Frobenius norm $\|\M K - V\|_2^2$. Exact memorization requires solving the linear system $\M K = V$.

Vectorizing the expression yields the system $(K^\top \otimes I_{d_v}) \mathrm{vec}(\M) = \mathrm{vec}(V)$, which has $m d_v$ scalar equations in $d_k d_v$ unknowns. When the keys are linearly independent, $\operatorname{rank}(K) = m$, and hence the system matrix has full row rank $m d_v$. Solvability thus requires $m d_v \le d_k d_v$, or equivalently $m \le d_k$. This matches classic results on the storage capacity of linear associative memories such as the Willshaw model and Hopfield networks, where capacity is tied to the rank of the input embedding \citep{willshaw1969non, hopfield1982neural}.

When $m \le d_k$ and $K$ has full column rank, one can construct an exact interpolating solution via the Moore–Penrose pseudoinverse: $\M^* = V K^\top$. Then $\M^* K = V K^{\top} K = V$, achieving zero training error. Thus the upper bound is tight.

Moreover, full-batch gradient descent on this objective with step size $0 < \eta < 2 / \lambda_{\max}(K K^\top)$ yields iterates $\M_{t+1} = \M_t - \eta (\M_t K - V) K^\top$, which converge to the minimum-norm interpolating solution $\M^\dagger = V K^\top$ when $m \le d_k$. This is a well-known implicit bias of gradient descent in overparameterized linear models \citep{satpathi2021dynamics}.

Finally, the same rank-based constraint governs the capacity of linear or multi-head attention modules. In such architectures, the output context matrix has rank at most $\operatorname{rank}(K) \le d_k$, which directly limits their expressivity. Recent analyses identify this “low-rank bottleneck” as a capacity-limiting effect in Transformers \citep{bhojanapalli2020low}.
\end{proof}

\DeepMemory*
Early theoretical works established that even simple network architectures can memorize a significant number of input-output mappings, with capacity often related to the number of network parameters (e.g., weights and biases) and the input dimensionality \citet{cover1965geometricalproperties, baum1988capabilitiesmultilayer,  huang2003learningcapability}. For instance, ~\citet{baum1988capabilitiesmultilayer} demonstrated that $\left\lceil \frac{N}{d} \right\rceil$ neurons are sufficient for a single-hidden-layer network with threshold units to memorize $N$ input-label pairs from $\mathbb{R}^d$.

Networks employing Rectified Linear Units (ReLUs), exhibit a piecewise affine behavior. The input space is partitioned into numerous linear regions, and within each region, the network computes a distinct affine transformation~\citet{montufar2014numberlinearregions, pascanu2014numberresponseregionsdeep}. This structure is pivotal for analyzing their expressive power and storage capacity. The precise relationship between depth, width, the number of linear regions, and the ultimate capacity to store specific key-value associations, especially with constraints like linearly independent keys, remains an active area of research.

\begin{proof}

Let $m$ denote the number of $(\vk_i, \vv_i)$ pairs memorized exactly by $\M$, and assume the keys $\{\vk_i\}_{i=1}^m \subset \R^{d_k}$ are linearly independent. Let $d_h^{(0)} := d_k$, $d_h^{(\cL_\M)} := d_v$, and for each layer $1 \le \ell \le \cL_\M$, define $W^{(\ell)} \in \R^{d_h^{(\ell)} \times d_h^{(\ell-1)}}$. Biases are omitted for simplicity.

Since $\sigma(x) = \max(0,x)$ is piecewise linear, the composition of linear maps and ReLU activations yields a piecewise affine function. For any fixed activation pattern (i.e., fixed sign of pre-activations), the MLP acts as:

\begin{align*}
\M(\cdot) = A \cdot + B, \quad \text{where } A = W^{(\cL_\M)} D^{(\cL_\M - 1)} W^{(\cL_\M - 1)} \cdots D^{(1)} W^{(1)},
\end{align*}

and each $D^{(\ell)}$ is a diagonal $\{0,1\}$ matrix selecting the active units. Therefore, when all keys fall into the same linear region (which occurs generically after a small perturbation), $\M$ is a single affine transformation.

Let $\mathbf{K} := [\vk_1\; \cdots\; \vk_m] \in \R^{d_k \times m}$ and $\mathbf{V} := [\vv_1\; \cdots\; \vv_m] \in \R^{d_v \times m}$. Exact memorization implies $A \mathbf{K} = \mathbf{V}$, so:

\begin{align*}
\rank(\mathbf{V}) \le \rank(A), \quad m = \rank(\mathbf{K}) \le \min\{\rank(A), d_k\}.
\end{align*}

Now observe:
\begin{align*}
A = W^{(\cL_\M)} \underbrace{D^{(\cL_\M-1)} W^{(\cL_\M -1)}}_{R_{\cL_\M -1}} \cdots \underbrace{D^{(1)} W^{(1)}}_{R_1},
\end{align*}
and thus the rank of $A$ is bounded by the minimal width encountered along each path times the immediate input dimension:
\begin{align*}
\rank(A) \le \sum_{i=1}^{\cL_\M} \left( \min_{j \ge i} d_h^{(j)} \right) d_h^{(i)} 
= \mathcal{O}\left( d_k d_v \sum_{i=1}^{\cL_\M} \min_{j \ge i} d_h^{(j)} d_h^{(i+1)} \right).
\end{align*}

Hence,
\begin{align*}
m \le \mathcal{O}\left( d_k d_v \sum_{i=1}^{\cL_\M} \min_{j \ge i} d_h^{(j)} d_h^{(i+1)} \right)
\end{align*}
\end{proof}
\PolyCapacity*
\begin{proof}
Let us begin by analyzing the dimension of the lifted feature space induced by $\phi_p$. A monomial in $d_k$ variables of total degree exactly $\ell$ has the form $\vk^\alpha = \prod_{j=1}^{d_k} k_j^{\alpha_j}$, where $\alpha \in \mathbb{N}^{d_k}$ and $|\alpha| := \sum_{j=1}^{d_k} \alpha_j = \ell$. The number of such monomials is given by the classical stars-and-bars formula, which counts the number of integer solutions to $\alpha_1 + \cdots + \alpha_{d_k} = \ell$, yielding
\begin{align*}
  \binom{d_k + \ell - 1}{\ell}.
\end{align*}
Summing over all degrees $\ell = 0$ to $p$ gives the total number of monomials (i.e., the output dimension of $\phi_p$),
\begin{align*}
  D = \sum_{\ell=0}^p \binom{d_k + \ell - 1}{\ell} = \binom{d_k + p}{p},
\end{align*}
where the final identity follows from the hockey-stick identity in combinatorics.

To characterize the memorization capacity, we reformulate the loss in matrix notation. Let $\Phi := [\phi_p(\vk_1)\;\cdots\;\phi_p(\vk_m)] \in \mathbb{R}^{D \times m}$ and $V := [\vv_1\;\cdots\;\vv_m] \in \mathbb{R}^{d_v \times m}$. Then the objective becomes
\begin{align*}
  L(\M) = \tfrac{1}{2} \|\M \Phi - V\|_2^2.
\end{align*}
Exact memorization corresponds to the existence of a matrix $\M$ such that $\M \Phi = V$. This is a linear system in which $\M$ acts on the columns of $\Phi$, so the rank of $\Phi$ necessarily limits the number of independent targets $\vv_i$ that can be fitted exactly.

By the sub-multiplicativity of rank, for any matrices $A$ and $B$, we have
\begin{align*}
  \operatorname{rank}(AB) \le \min\{\operatorname{rank}(A), \operatorname{rank}(B)\}.
\end{align*}
Applying this to $\M \Phi$ yields
\begin{align*}
  \operatorname{rank}(\M \Phi) \le \operatorname{rank}(\Phi) \le D.
\end{align*}
Now consider a case where the targets $\vv_1, \dots, \vv_m$ are linearly independent; for instance, take $V = [e_1, \dots, e_m]$, the first $m$ standard basis vectors. Then $\operatorname{rank}(V) = m$. If $m > D$, we necessarily have $\operatorname{rank}(\M \Phi) < \operatorname{rank}(V)$ for every choice of $\M$, implying that the system $\M \Phi = V$ is unsolvable. Hence, the loss remains strictly positive, and exact memorization is impossible.

This establishes that no method, regardless of optimization procedure, can memorize more than $D = \binom{d_k + p}{p}$ independent input-output pairs under a degree-$\le p$ polynomial lifting. Since $\binom{d_k + p}{p} = \Theta(d_k^p)$ for fixed $p$, the result follows: the memorization capacity is bounded above by $\mathcal{O}(d_k^p)$.
\end{proof}

\section{Detailed Formulations of All Architectures}\label{app:all-arch}
In this section, for the sake of clarity, we discuss the details of all architectures that we discuss through the paper:

\subsection{Deep Linear Attention (DLA)}
We design Deep Linear Attention (DLA)—linear attention module that uses a deep MLP as the memory (KV cache)—as one of the baselines of this study. Given input $\vx \in \R^{N \times  d_{\text{in}}}$, we project the input into matrices of keys, values and queries:
\begin{align}
    \mb{Q} = \begin{pmatrix} \vq_1 \\ \vdots \\ \vq_N
    \end{pmatrix}= \vx \mb{W}_{Q}, \qquad \quad \mb{K} = \begin{pmatrix} \vk_1 \\ \vdots \\ \vk_N
    \end{pmatrix} = \vx \mb{W}_{K}, \qquad \quad  \mb{V} = \begin{pmatrix} \vv_1 \\ \vdots \\ \vv_N
    \end{pmatrix} = \vx \mb{W}_{V},
\end{align}
where $\mb{W}_{Q}, \mb{W}_{K},$ and $\mb{W}_{V}$ are learnable linear layers. We then define memory as  a learning module that optimizes the inner-dot product similarity using gradient descent: i.e., 
\begin{align}
    \min_{\M} \undermath{\ell(\M_{t-1}; \vk_t, \vv_t)}{\inner{\M(\vk_t)}{\vv_t}}.
\end{align}
The above optimization using gradient descent results in the following recurrence (we also add weight decay with input-dependent parameter $\alpha_t$):
\begin{align}
    \M_{t} = \alpha_t \M_{t-1}  - \eta_t \nabla \ell(\M_{t-1}; \vk_t, \vv_t)
\end{align}
which in the case of linear memory (i.e., $\M_t = W_t \in \R^{d \times d}$) it becomes:
\begin{align}
    W_{t} = \alpha_t W_{t-1} + \vv_t \vk_t^{\top},
\end{align}
which is the formulation of gated linear attention. We use the same training process as other models (see \autoref{sec:parallelization}).

\subsection{Sliding Window Linear Attention (SWLA)}
The design of SWLA is the same as the design of DLA, but with the use of sliding window objective. That is, given keys, values, and queries:
\begin{align}
    \mb{Q} = \begin{pmatrix} \vq_1 \\ \vdots \\ \vq_N
    \end{pmatrix}= \vx \mb{W}_{Q}, \qquad \quad \mb{K} = \begin{pmatrix} \vk_1 \\ \vdots \\ \vk_N
    \end{pmatrix} = \vx \mb{W}_{K}, \qquad \quad  \mb{V} = \begin{pmatrix} \vv_1 \\ \vdots \\ \vv_N
    \end{pmatrix} = \vx \mb{W}_{V},
\end{align}
we optimize the internal objective of:
\begin{align}
    \min_{\M} \undermath{\ell(\M_{t-1}; \vk_t, \vv_t)}{\sum_{i = t - c + 1}^{t}\inner{\M_{t-1}(\vk_i)}{\vv_i}}.
\end{align}
The above formulation, results in:
\begin{align}
     \M_{t} = \alpha_t \M_{t-1}  - \nabla \ell(\M_{t-1}; \vk_t, \vv_t) = \alpha_t \M_{t-1}  - \sum_{i = t - c + 1}^{t}  \eta^{(t)}_i \nabla \inner{\M_{t-1}(\vk_i)}{\vv_i},
\end{align}
which in the case of linear memory (i.e., $\M_t = W_t \in \R^{d \times d}$) it becomes:
\begin{align}
    \M_{t} = \alpha_t \M_{t-1}  - \sum_{i = t - c + 1}^{t} \eta^{(t)}_i  \vv_i \vk_i^{\top}.
\end{align}

\subsection{\omodel}
In the design of \omodel, we use replace the dot-prodcut similarity objective with $\ell(\M_{t-1}; \vk_t, \vv_t) = \sum_{i = t - c + 1}^{t} \| \M_{t-1}(\phi(\vk_i)) - \vv_i \|^{2}_2$ ,which results in the recurrence of:
\begin{align}
     \M_{t} = \alpha_t \M_{t-1}  - \nabla \ell(\M_{t-1}; \vk_t, \vv_t) = \alpha_t \M_{t-1}  - \sum_{i = t - c + 1}^{t}  \eta^{(t)}_i \nabla \| \M_{t-1}(\phi(\vk_i)) - \vv_i \|^{2}_2.
\end{align}
In the above formulation, $\phi(.)$ is the polynomial feature mapping function.

\subsection{\model}
In the \model, we use the same internal objective as \omodel{} but we optimize it using Muon optimizer~\citep{jordan2024muon} with weight decay. That is, 
\begin{align}
    \M_{t} &= \alpha_t\M_{t-1} + \texttt{Newton-schulz5}(\SSS_t)\\
    \SSS_t &= \theta_t \SSS_{t-1} - \sum_{i = t - c + 1}^{t}  \eta^{(t)}_i \nabla \| \M_{t-1}(\phi(\vk_i)) - \vv_i \|^{2}_2.
\end{align}

\section{Experimental Details} \label{app:exp-details}
In our experimental setup we follow recent studies on linear recurrent models~\citep{yang2024gated, behrouz2024titans, behrouz2025Miras}, we use Wikitext~\citep{merity2017pointer}, LMB~\citep{paperno-etal-2016-lambada}, PIQA~\citep{bisk2020piqa}, HellaSwag~\citep{zellers-etal-2019-hellaswag}, WinoGrande~\citep{sakaguchi2021winogrande},  ARC-easy (ARC-e) and ARC-challenge (ARC-c)~\citep{clark2018think}, SIQA~\citep{sap-etal-2019-social}, and BoolQ~\citep{clark-etal-2019-boolq}. Also, the baselines results are from \citet{behrouz2025Miras, behrouz2024titans}. In the training, we use T5 tokenizer with a vocabulary size of 32K and use training length of 4K tokens (2K for SWA). We employ AdamW optimizer with learning rate of $4e$-$4$ with cosine annealing schedule with batch size of 0.5M tokens, and weight decay of $0.1$. The architectural details are also reported in \autoref{tab:exp-details}. The baseline results for 1.3B are from \citet{yang2024gated} and for 760M are from
\citet{behrouz2024titans, behrouz2025Miras}.

For the memory architecture, unless state otherwise, we use an MLP with $2$ layers with expansion factor of 4 and GELU activation function~\citep{hendrycks2016gaussian}. We also use residual connections and layer norm at the end of each chunk: $\M(x) = x + W_1 \sigma(W_2 x)$.

\begin{table*}
    \centering
    \caption{Architectural Details.}
    \label{tab:exp-details}
    \resizebox{0.5\linewidth}{!}{
    \begin{tabular}{c c c c c c}
    \toprule
         Model    & Block & Dim & Head & Peak LR & Token\\
         \midrule
         \midrule
        170M & 12 & 768 & 16 & 3e-3 & 15B\\
        340M & 24 & 1024 & 16 & 1.5e-3 & 15B\\
        760M & 24 & 1536 & 16 & 1.25e-3 & 30B\\
        1.3B & 18 & 2048 & 8 & 7e-4 & 100B \\
    \toprule
    \end{tabular}
    }
\end{table*}

\section{Additional Experimental Results} \label{app:exp}
In this section,  we provide additional experimental results to support the design of our models, understand the effect of different components and also evaluate their performance in long context, in-context recall and MAD tasks.

\subsection{Language Modeling and Common-sense Reasoning (Small Scale)}
In \autoref{sec:experiments} we presented a subset of results on language modeling and common-sense reasoning tasks. In this section, we further report the results for all scales of models. The results are in \autoref{tab:lm_results_full}. 

\head{State-of-the-art Results} 
Looking at the performance of \model{} and \omodel, both architectures perform favorably compared to modern linear recurrent models and Transformers, achieving lower perplexity and better accuracy in downstream tasks. Even the fully recurrent version of these models outperform hybrid models such as Samba~\citep{ren2024samba} and Gated DeltaNet-H2~\citep{yang2024gated}. Using the hybrid variants of MAG and MAL further improve the performance of \model, which shows the complementary role of recurrent long-term memory and attention.

\head{The Effect of Design}
Comparing the performance of \model, \omodel, and baselines SWLA and DLA, we can see the role of $\ell_2$ regression loss as the attentional bias. Also, the better performance of SWLA compared to GLA and RetNet indicates the importance of memorizing the context, instead of memorizing individual tokens.

\begin{table*}[t!]
\centering
\caption{
Performance of \model{} and baselines on language modeling and common-sense reasoning tasks. The best results are highlighted \colorbox{myblue}{highlighted}. 
}\label{tab:lm_results_full}
\centering
\resizebox{0.9\linewidth}{!}{
\centering
\begin{tabular}{l|c c|c c c c c c c c c}
\toprule
\textbf{Model}  & \textbf{Wiki.}  &  \textbf{LMB.} &  \textbf{LMB.} & \textbf{PIQA} &    \textbf{Hella.} & \textbf{Wino.} & \textbf{ARC-e} &  \textbf{ARC-c} &  \textbf{SIQA}  & \textbf{BoolQ} &  \textbf{Avg.} \\
 & ppl $\downarrow$  &  ppl $\downarrow$  &  acc $\uparrow$  & acc $\uparrow$ &   acc\_n $\uparrow$  & acc $\uparrow$  & acc $\uparrow$ & acc\_n $\uparrow$ &  acc $\uparrow$  & acc $\uparrow$ &   $\uparrow$  \\
\midrule
\midrule
\multicolumn{12}{c}{340M params / 15B tokens} \\
\midrule
 Transformer++ & 31.52 & 41.08 &  30.76 & 62.98  &  34.76 & 50.53  & 45.21  & 24.05 & 36.81 & 58.24 & 42.92\\
 RetNet & 32.50 & 49.73 & 28.24 & 62.61 & 34.15 &  50.91 & 44.27 & 23.62 & 36.79 & 59.72 & 42.54\\
 GLA & 28.51 & 43.02 & 28.73 & 64.05 & 35.96 & 50.00 & 54.19 & 24.29 & 37.13 & 58.39 & 44.09\\
 Mamba & 30.83 & 40.21 & 29.94 & 63.79 & 35.88 & 49.82 & 49.24 & 24.56 &  35.41  & 60.07 & 43.59\\
 DeltaNet & 28.65 & 47.30 & 28.43 & 63.52 & 35.95 & 49.63 & 52.68 & 25.37 &   {37.96}  &  58.79  & 44.04 \\
 TTT & 27.44 & 34.19 & 30.06 & 63.97  & 35.71 & 50.08 & 53.01 & 26.11 & 37.32 & 59.83 & 44.51\\
 Gated DeltaNet & 27.01 & 30.94 &  34.11 & 63.08 & 38.12  &   51.60  &  55.28  &  26.77  & 34.89 & 59.54  & 45.42\\
\textsc{Moneta} & 26.19 &  29.31 &   35.70  & 63.99  & 39.23  & 52.04  &  55.96  &  27.15 & 37.29 &   60.22 &   46.44 \\
\textsc{Yaad} & 26.61 &   29.11 &  34.09  & 64.93  &  39.86  & 51.12  & 54.75  &  28.64 & 33.82 & 60.29 & 45.93 \\
\textsc{Memora} & 27.16 &  30.44 &  33.68  &  65.21  & 39.17  & 51.23  & 53.40  & 27.99 & 34.1& 59.29 & 45.51 \\
\midrule
DLA (ours) & 27.93 & 35.09 & 30.8 & 62.9 & 36.2  & 50.4 & 53.5 & 26.7 & 37.1 & 59.7 & 44.76 \\
SWDT (ours) & 26.98 & 33.95 & 32.4 & 63.1 & 38.2 & 50.9 & 54.9 & 25.9 & 37.5 & 59.6 & 45.31\\
\midrule
\omodel{} (ours) & 26.03 & 28.76  & 35.6 & 65.3 & 39.7 & 52.0 & 56.1 & 28.6 & 37.7 & 60.4 & 46.93\\
\model{} (ours) & 25.88  & 28.54  & 36.1 & 64.9 & 40.1 & 52.7 & 56.4 & 28.8 & 38.1 & 61.2 & 47.28\\
\bottomrule
\end{tabular}
}
\end{table*}


\end{document}